\documentclass[11pt]{article}
\usepackage[sc]{mathpazo} 
\usepackage{environ}

\usepackage[margin=1.15in]{geometry}
\usepackage[utf8x]{inputenc}
\usepackage[usenames,dvipsnames]{xcolor}
\definecolor{Gred}{RGB}{219, 50, 54}
\definecolor{Ggreen}{RGB}{60, 186, 84}
\definecolor{Gblue}{RGB}{72, 133, 237}
\definecolor{Gyellow}{RGB}{247, 178, 16}
\definecolor{ToCgreen}{RGB}{0, 128, 0}
\definecolor{myGold}{RGB}{231,141,20}
\definecolor{myBlue}{rgb}{0.19,0.41,.65}
\definecolor{myPurple}{RGB}{175,0,124}
\definecolor{niceRed}{RGB}{153,0,0}
\definecolor{niceRed}{RGB}{190,38,38}
\definecolor{blueGrotto}{HTML}{059DC0}
\definecolor{royalBlue}{HTML}{057DCD}
\definecolor{navyBlueP}{HTML}{0B579C}
\definecolor{limeGreen}{HTML}{81B622}

\usepackage{cmap}
\usepackage[T1]{fontenc}
\usepackage{bm}
\usepackage{amsmath}
\usepackage{amsfonts}
\usepackage{amssymb}
\usepackage{amsbsy}
\usepackage{amsthm}

\usepackage{graphicx, ucs}

\usepackage{subcaption}
\usepackage{rotating}
\usepackage{float}
\usepackage{tikz}

\usepackage{algorithm, caption}
\usepackage[noend]{algpseudocode}
\usepackage{listings}

\usepackage{enumitem}

\usepackage[pagebackref]{hyperref}
\hypersetup{
	colorlinks = true,
	urlcolor = {myPurple},
	linkcolor = {royalBlue},
	citecolor = {ForestGreen}
}
\usepackage[nameinlink]{cleveref}

\usepackage{multirow}
\usepackage{array}

\usepackage{chngcntr}
\counterwithin{equation}{section}

\usepackage{soul}
\usepackage{dsfont}
\usepackage{nicefrac}

\usepackage{fancyhdr}
\fancyheadoffset{0pt}
\def\compactify{\itemsep=0pt \topsep=0pt \partopsep=0pt \parsep=0pt}
\let\latexusecounter=\usecounter

\definecolor{myC}{rgb}{0, 255, 255}
\definecolor{myY}{rgb}{204, 204, 0}
\definecolor{myM}{rgb}{255, 0, 255}
\definecolor{secinhead}{RGB}{249,196,95}
\definecolor{lgray}{gray}{0.8}

\newtheorem{theorem}{Theorem}  
\newtheorem{proposition}{Proposition}
\newtheorem{corollary}{Corollary}
\newtheorem{lemma}{Lemma}
\newtheorem{definition}{Definition}
\newtheorem{inftheorem}{Informal Theorem}
\newtheorem{fact}{Fact}
\newtheorem{assumption}{Assumption}

\newcommand{\reals}{\mathbb{R}}

\newcommand{\D}{\mathcal{D}}
\newcommand{\Q}{\mathcal{Q}}

\newcommand{\poly}{\mathrm{poly}}

\def\l{\ell}
\def\<{\langle}
\def\>{\rangle}

\newcommand{\Be}{\textit{Be}}

\DeclareMathOperator*{\argmin}{argmin}

\def\poly{\mathrm{poly}}

\def\vec{\bm}

\def\D{\mathcal{D}}
\def\M{\mathcal{M}}
\def\U{\mathcal{U}}
\def\P{\mathcal{P}}
\def\Q{\mathcal{Q}}
\def\poly{\mathrm{poly}}
\def\eps{\varepsilon}
\def\Prob{\Pr}
\def\Exp{\E}
\def\E{\mathop{\mathbb{E}}}
\def\reals{\mathbb{R}}
\def\flip{\textsc{flip}}
\def\Be{\mathcal{B}e}
\def\eqdef{\stackrel{\mathrm{def}}{=}}

\def\S{\mathfrak{S}}
\def\ln{\log}
\def\truez{\vec{z}^{\star}}
\def\truep{\vec{p}^{\star}}
\def\ball{\mathcal{B}}
\def\rad{B}

\def\hat{\widehat}

\makeatletter
\renewenvironment{abstract}{%
	\if@twocolumn
	\section*{\abstractname}%
	\else 
	\begin{center}%
		{\bfseries \large\abstractname\vspace{\z@}}
	\end{center}%
	\quotation
	\fi}
{\if@twocolumn\else\endquotation\fi}
\makeatother

\begin{document}
\footnotetext{Accepted at the Conference on Learning Theory (COLT) 2020.}

	\title{Efficient Parameter Estimation of\\ Truncated Boolean Product Distributions}
	\author{
		\textbf{Dimitris Fotakis}\footnote{National Technical University of Athens, \url{fotakis@cs.ntua.gr}} \\
		\small NTUA 
		\and
		\textbf{Alkis Kalavasis} 
		\footnote{National Technical University of Athens, \url{kalavasisalkis@mail.ntua.gr}}
		\\
		\small NTUA 
		\and
		\textbf{Christos Tzamos}
		\footnote{ University of Wisconsin-Madison, \url{tzamos@wisc.edu}}
		\\
		\small UW Madison \\
	}
	\maketitle
	\thispagestyle{empty}

	\begin{abstract}
	\small 
We study the problem of estimating the parameters of a Boolean product distribution in $d$ dimensions, when the samples are truncated by a set $S \subset \{0, 1\}^d$ accessible through a membership oracle. This is the first time that the computational and statistical complexity of learning from truncated samples is considered in a discrete setting. 

We introduce a natural notion of \emph{fatness} of the truncation set $S$, under which truncated samples reveal enough information about the true distribution. We show that if the truncation set is sufficiently fat, samples from the true distribution can be generated from truncated samples. A stunning consequence is that virtually any statistical task (e.g., learning in total variation distance, parameter estimation, uniformity or identity testing) that can be performed efficiently for Boolean product distributions, can also be performed from truncated samples, with a small increase in sample complexity. We generalize our approach to ranking distributions over $d$ alternatives, where we show how fatness implies efficient parameter estimation of Mallows models from truncated samples. 

Exploring the limits of learning discrete models from truncated samples, we identify three natural conditions that are necessary for efficient identifiability: (i) the truncation set $S$ should be rich enough; (ii) $S$ should be accessible through membership queries; and (iii) the truncation by $S$ should leave enough randomness in all directions. By carefully adapting the Stochastic Gradient Descent approach of (Daskalakis et al., FOCS 2018), we show that {conditions (i), (ii), and a strengthening of (iii)} are also sufficient for efficient learning of truncated Boolean product distributions.
\end{abstract}
	
\clearpage

\section{Introduction}
\label{s:intro}

Parameter estimation and learning from truncated samples is an important and challenging problem in Statistics. The goal is to estimate the parameters of the true distribution based only on samples that fall within a (possibly small) subset $S$ of the distribution's support.

Sample truncation occurs naturally in a variety of settings in science, engineering, economics, business and social sciences. Typical examples include selection bias in epidemiology and medical studies, and anecdotal ``paradoxes'' in damage and injury analysis explained by survivor bias. Statistical estimation from truncated samples goes back to at least \cite{Galton1897}, who analyzed truncated samples corresponding to speeds of American trotting horses, and includes classical results on the use of the moments method \cite{PearsonLee1908,Lee1914} and the maximum likelihood method \cite{fisher31} for estimating a univariate Gaussian distribution from truncated samples (see also \cite{DGTZ18} for a detailed discussion on the history and the significance of statistical estimation from truncated samples). 

In the last few years, there has been an increasing interest in computationally and statistically efficient algorithms for learning multivariate Gaussian distributions from truncated samples (when the truncation set is known \cite{DGTZ18} or unknown \cite{KTZ19}) and for training linear regression models using truncated (or censored) data \cite{DGTZ19}. In addition to the elegant and powerful application of Stochastic Gradient Descent to optimizing a seemingly unknown maximum likelihood function from truncated samples, a significant contribution of \cite{DGTZ18,KTZ19,DGTZ19} concerns necessary conditions for efficient statistical estimation of multivariate Gaussian or regression models from truncated samples. More recently, \cite{NP19} showed how to use Expectation-Maximization for learning mixtures of two Gaussian distributions from truncated samples. 

Despite the strong results above on efficient learning from truncated samples 
for continuous settings, we are not aware of any previous work on learning discrete models from truncated samples. We note that certain elements of the prior approaches in inference from truncated data are inherently continuous and it is not clear to which extent (and under which conditions) can be adapted to a discrete setting. E.g., statistical estimation from truncated samples in a discrete setting should deal with a situation where the truncation removes virtually all randomness from certain directions, something that cannot be the result of nontrivial truncations 
in a continuous setting. 

\paragraph{Our Setting.}
Motivated by this gap in relevant literature, we investigate efficient parameter estimation of discrete models from truncated samples. We start with the fundamental setting of a Boolean product distribution $\D$ on the $d$-dimensional hypercube truncated by a set $S$, which is accessible through membership queries. The marginal of $\D$ in each direction $i$ is an independent Bernoulli distribution with parameter $p_i \in (0, 1)$. Our goal is to compute an estimation $\hat{\vec{p}}$ of the parameter vector $\vec{p}$ of $\D$ such that $\|\vec{p} - \hat{\vec{p}}\|_2 \leq \eps$, with probability of at least $1-\delta$, with time and sample complexity polynomial in $d$, $1/\eps$ and $\ln(1/\delta)$. We note that such an estimation $\hat{\vec{p}}$ (or an estimation $\hat{\vec{z}}$ of the logit parameters $\vec{z} = (\ln\frac{p_1}{1-p_1}, \ldots, \ln\frac{p_d}{1-p_d})$ of similar accuracy) implies an estimation of the true distribution within total variation distance $\eps$.  

\paragraph{Our Contributions.}
Significantly departing from the maximum likelihood estimation approach of \cite{DGTZ18,KTZ19,DGTZ19}, we introduce a natural notion of fatness of the truncation set $S$, under which samples from the truncated distribution $\D_S$ reveal enough information about the true distribution $\D$. Roughly speaking, a truncated Boolean product distribution $\D_S$ is \emph{$\alpha$-fat} in some direction $i$ of the Boolean hypercube, if for an $\alpha$ probability mass of the truncated samples, the neighboring sample with its $i$-th coordinate flipped is also in $S$. Therefore, with probability $\alpha$, conditional on the remaining coordinates, the $i$-th coordinate of a sample is distributed as the marginal of the true distribution $\D$ in direction $i$. So, if the truncated distribution $\D_S$ is $\alpha$-fat in all directions (e.g., the halfspace of all vectors with $L_1$ norm at most $k$ is a fat subset of the Boolean hypercube), a sample from $\D_S$ is quite likely to reveal significant information about the true distribution $\D$. Building on this intuition, we show how samples from the true distribution $\D$ can be generated from few truncated samples (see also \Cref{algo:fatsample}): 

\begin{inftheorem}
\label{thm:inf:fat}
With an expected number of $O(\ln(d)/\alpha)$ samples from the $\alpha$-fat truncation of a Boolean product distribution $\D$, we can generate a sample $\vec{x} \in \{0, 1\}^d$ distributed as in $\D$.  
\end{inftheorem}

We show (\Cref{lem:fat_necessary}) that fatness is also a necessary condition for \Cref{thm:inf:fat}. A stunning consequence of \Cref{thm:inf:fat} is that virtually any statistical task (e.g., learning in total variation distance, parameter estimation, sparse recovery, uniformity or identity testing)
that can be performed efficiently for a Boolean product distribution $\D$, can also be performed using truncated samples from $\D$, at the expense of a factor $O(\ln(d)/\alpha)$ increase in time and sample complexity. In \Cref{s:fatsets}, we obtain, as simple corollaries of \Cref{thm:inf:fat}, that the statistical tasks described in \cite{ADK15,DKS17,CDKS17,CKM+19} for Boolean product distributions can be performed using only truncated samples! 

To further demonstrate the power and the wide applicability of our approach, we extend the notion of fatness to the richer and more complex setting of ranking distributions on $d$ alternatives. In \Cref{s:ranking}, we show how to implement efficient statistical inference of Mallows models using samples from a fat truncated Mallows distribution (see \Cref{thm:mallows-fat}). 

Natural and powerful though, fatness is far from being necessary for efficient parameter estimation from truncated samples. Seeking a deeper understanding of the challenges of learning discrete models from truncated samples, we identify, in \Cref{s:conditions}, three natural conditions that we show to be necessary for efficient parameter estimation in our setting: 
\begin{description}
\item[\Cref{as:identify}:] The support of the distribution $\D$ on $S$ should be rich enough, in the sense that its truncation $\D_S$ should assign positive probability to a $\vec{x}^\star \in S$ and $d$ other vectors that remain linearly independent after we subtract $\vec{x}^\star$ from them. 

\item[\Cref{as:oracle}:] $S$ is accessible through a membership oracle that reveals whether $\vec{x} \in S$, for any $\vec{x}$ in the $d$-dimensional hypercube. 

\item[\Cref{as:anticon}:] The truncation of $\D$ by $S$ leaves enough randomness in all directions. More precisely, we require that in any direction $\vec{w} \in \reals^d$, any two samples from the truncated distribution $\D_S$ have sufficiently different projections on $\vec{w}$, with non-negligible probability. {Roughly speaking, for the truncated distribution $\mathcal{D}_S$, we assume that for any unit vector $w$, it holds that $\mathrm{Var}_{\mathcal{D}_S}(w^\top x) \geq \lambda$ for some parameter $\lambda  > 0.$}
\end{description}

\Cref{as:oracle} ensures that the learning algorithm has enough information about $S$ and is also required in the continuous setting. Without oracle access to $S$, for any Boolean product distribution $\D$, we can construct a (possibly exponentially large) truncation set $S$ such that sampling from the truncated distribution $\D_S$ appears identical to sampling from the uniform distribution, until the first duplicate sample appears (our construction is similar to \cite[Lemma~12]{DGTZ18}). 

Similarly to \cite{DGTZ18}, \Cref{as:oracle} is complemented by the additional natural requirement that the true distribution $\D$ should assign non-negligible probability mass to the truncation set $S$ (\Cref{as:mass}). The reason is that the parameter estimation algorithm evaluates the quality of its current estimation by generating samples in $S$ and comparing them with samples from $\D_S$. Assumptions~\ref{as:oracle}~and~\ref{as:mass} ensure that this can be performed efficiently. 

Assumptions~\ref{as:identify}~and~\ref{as:anticon} are specific to the discrete setting of the Boolean hypercube. \Cref{as:identify} requires that we should be able to normalize the truncation set $S$, by subtracting a vector $\vec{x}^\star$, so that its dimension remains $d$. If this is true, we can recover the parameters of a Boolean product distribution $\D$ from truncated samples by solving a linear system with $d$ equations and $d$ unknowns, which we obtain after normalization. We prove, in \Cref{lem:identify}, that \Cref{as:identify} is both sufficient and necessary for parameter recovery from truncated samples in our setting. 

\Cref{as:anticon} is a stronger version of \Cref{as:identify} and is necessary for efficient parameter estimation from truncated samples in the Boolean hypercube. It essentially requires that with sufficiently high probability, any set $X$ of polynomially many samples from $\D_S$ can be normalized, subtracting a vector $\vec{x}^\star$, so that $X$ includes a well-conditioned $d \times d$ matrix, after normalization. 

Beyond showing that these assumptions are necessary for efficient identifiability, we show that they are also sufficient {(for some regime of anti-concentration $\lambda)$} and provide a computationally efficient algorithm for learning Boolean product distributions. Our algorithm is based on a careful adaptation of the approach of \cite{DGTZ18} which uses Stochastic Gradient Descent on the negative log-likelihood. While the analysis consists of the same conceptual steps as that of \cite{DGTZ18}, it requires dealing with a number of technical details that arise due to discreteness. One technical contribution of our work is using the assumptions for identifiability to establish strong-convexity of the negative log-likelihood in a small ball around the true parameters (see \Cref{lem:str-conv} and \Cref{lem:aconc-ball}). Our main result is that: 
\begin{inftheorem}
\label{thm:inf:sgd}
Under Assumptions~\ref{as:identify}\,-\,\ref{as:mass}, \Cref{algo:sgd} computes an estimation $\hat{\vec{z}}$ of the logit vector $\vec{z}$ of the true distribution $\D$ such that $\|\vec{z} - \hat{\vec{z}}\|_2 \leq \eps$ with probability at least $1-\delta$, and achieves time and sample complexity polynomial in $d$, $1/\eps,~{\exp(1/\lambda)}$, and $\ln(1/\delta)$.
\end{inftheorem}

\paragraph{Related Work.}
Our work develops novel techniques for truncated statistics for discrete distributions. As aforementioned, there has been a large number of recent works dealing inference with truncated data from a Gaussian distribution~\cite{DGTZ18,KTZ19,DGTZ19} or mixtures of Gaussians~\cite{NP19} but to the best of our knowledge there is no work dealing with discrete distributions. An additional feature of our work compared to those results is that our methods are not limited to parameter estimation but enable any statistical task to be performed on truncated datasets by providing a sampler to the true underlying distribution. While this requires a mildly stronger than necessary but natural assumption on the truncation set, we show that the more complex SGD based methods developed in prior work can also be applied in the discrete settings we consider.

The field of robust statistics is also very related to our work as it also
deals with biased data-sets and aims to identify the distribution that generated the data.
Truncation can be seen as an adversary erasing
samples outside a certain set. Recently, there has been a lot of
theoretical work for computationally-efficient robust estimation of high-dimensional distributions in the presence of arbitrary
corruptions to a small $\varepsilon$ fraction of the samples, allowing
for both deletions of samples and additions of samples
\cite{DKK+16b,CSV17,LRV16,DKK+17,DKK+18,hopkins2019hard}. In particular, the work of~\cite{DKK+16b}
deals with the problem of learning binary-product distributions.

Another line of related work concerns learning from positive examples. The work of \cite{de2014learning} considers a setting where samples are obtained from the uniform distribution over the hypercube truncated on a set $S$. However, their goal is somewhat orthogonal to ours. It aims to accurately learn the uniform distribution conditioned on the set $S$. In contrast, in our setting we have membership access to the truncation set  and the goal is to learn the parameters of the target Boolean product distribution.
More recently, \cite{canonne2020learning} extend the results of \cite{de2014learning} to the continuous domain.

Another related literature within learning theory aims to learn discrete distributions through conditional samples. In the conditional sampling model that was recently introduced concurrently by~\cite{ChakrabortyFGM13} and ~\cite{CanonneRS14,CanonneRS15}, the goal is again to learn an underlying discrete distribution through conditional/truncated samples but the learner can change the truncation set on demand. This is known to be a more powerful model for distribution learning and testing than standard sampling~\cite{Canonne15b,FalahatgarJOPS15,AcharyaCK15b,BhattacharyyaC18,AcharyaCK15a,GouleakisTZ17,KamathT19,canonne2019random}.

\section{Preliminaries}
\label{s:prelim}

We use lowercase bold letters $\vec{x}$ to denote $d$-dimensional vectors. We let $\|\vec{x}\|_{p} = (\sum_{i=1}^{d}|x_{i}|^{p})^{1/p}$ denote the $L_p$ norm and $\|\vec{x}\|_{\infty} = \max_{i \in [d]} \{ |x_{i}| \}$ denote the $L_\infty$ norm of a vector $\vec{x}$. We let $[d] \eqdef \{1, \ldots, d\}$ and $\Pi_{d} = \{0 , 1\}^d$ denotes the $d$-dimensional Boolean hypercube. 

For any vector $\vec{x}$, $\vec{x}_{-i}$ is the vector obtained from $\vec{x}$ by removing the $i$-th coordinate and $(\vec{x}_{-i}, y)$ is the vector obtained from $\vec{x}$ by replacing $x_i$ by $y$. Similarly, given a set $S \subseteq \Pi_{d}$, we let $S_{-i} = \{ \vec{x}_{-i} : (\vec{x}_{-i}, 0) \in S \lor (\vec{x}_{-i}, 1) \in S \}$ be the projection of $S$ to $\Pi_{[d] \setminus \{ i \}}$. For any $\vec{x} \in \Pi_d$ and any coordinate $i \in [d]$, we let $\flip(\vec{x},i) = (\vec{x}_{-i}, 1-x_i)$ denote $\vec{x}$ with its $i$-th coordinated flipped. 

\paragraph{Bernoulli Distribution.} 
For any $p \in [0,1]$, we let $\Be(p)$ denote the Bernoulli distribution with parameter $p$. For any $x \in \{0,1\}$, $\Be(p;x) = p^{x}(1-p)^{1-x}$
denotes the probability of value $x$ under $\Be(p)$. The Bernoulli distribution is an exponential family%
\footnote{The exponential family $\mathcal{E}(\vec{T}, h)$ with sufficient statistics $\vec{T}$, carrier measure $h$ and natural parameters $\vec{\eta}$ is the family of distributions $\mathcal{E}(\vec{T}, h) = \{\mathcal{P}_{\vec{\eta}} : \vec{\eta} \in \mathcal{H}_{\vec{T}, h}\}$, where the probability distribution $\mathcal{P}_{\vec{\eta}}$ has density $p_{\vec{\eta}}(x) = h(x)\exp(\vec{\eta}^{T}\vec{T}(x) - \Lambda(\vec{\eta}))$, where $\Lambda$ is the log-partition function.},
where the natural parameter, denoted $z$, is the logit $z = \ln \frac{p}{1-p}$ of the parameter $p$\,\footnote{The base of the logarithm function $\log$ used throughout the paper is insignificant.}. 
The inverse parameter mapping is $p = \frac{1}{1+\exp(-z)}$. Also, the base measure is $h(x) = 1$, the sufficient statistic is the identity mapping $T(x) = x$ and the log-partition function with respect to $p$ is $\Lambda(p) = -\ln(1-p)$.

\paragraph{Boolean Product Distribution.} 
We mostly focus on the fundamental family of \emph{Boolean product distributions} on the $d$-dimensional hypercube $\Pi_{d}$. A Boolean product distribution with parameter vector $\vec{p} = (p_1, \ldots, p_d)$, usually denoted by $\D(\vec{p})$, is the product of $d$ independent Bernoulli distributions, i.e., $\D(\vec{p}) = \Be(p_1) \otimes \cdots \otimes \Be(p_d)$. The Boolean product distribution can be expressed in the form of an exponential family as follows: 
\begin{equation}\label{eq:expofamily}
    \D(\vec{z} ; \vec{x}) = \frac{\exp(\vec{x}^{T}\vec{z})}{\prod_{i \in [d]}(1+\exp(z_{i}))}\,, 
\end{equation}
where $\vec{z} = (z_1, \ldots, z_d)$ is the natural parameter vector with $z_i = \ln \frac{p_{i}}{1-p_{i}}$ for each $i \in [d]$. 

In the following, we always let $\D$ (or $\D(\vec{p})$ or $\D(\vec{z})$, when we want to emphasize the parameter vector $\vec{p}$ or the natural parameter vector $\vec{z}$) denote a Boolean product distribution. We denote $\vec{z}(\vec{p})$ (or simply $\vec{z}$, when $\vec{p}$ is clear from the context) the vector of natural parameters of $\mathcal{D}$. We let $\D(\vec{p}; \vec{x})$ and $\D(\vec{z}; \vec{x})$ (or simply $\D(\vec{x})$, when $\vec{p}$ or $\vec{z}$ are clear from the context) denote the probability of $\vec{x} \in \Pi_d$ under $\D$. 
Given a subset $S \subseteq \Pi_{d}$ of the Boolean hypercube, the probability mass assigned to $S$ by a distribution $\D(\vec{p})$, usually denoted $\D(\vec{p}; S)$ (or simply $\D(S)$, when $\vec{p}$ is clear from the context), $\D(\vec{p};S) = \sum_{\vec{x} \in S}\D(\vec{p};\vec{x})$.

\paragraph{Truncated Boolean Product Distribution.} 
Given a Boolean product distribution $\D$, we define the \emph{truncated Boolean product distribution} $\D_S$, for any fixed $S \subseteq \Pi_{d}$. $\D_S$ has $\D_S(\vec{x}) = \D(\vec{x})/\D(S)$, for all $\vec{x} \in S$, and $\D_S(\vec{x}) = 0$, otherwise. We often refer to $\D_S$ as the truncation of $\D$ (by $S$) and to $S$ as the \emph{truncation set}. 

It is sometimes convenient (especially when we discuss assumptions~\ref{as:identify}~and~\ref{as:anticon}, in \Cref{s:conditions}), to refer to some fixed element of $S$. We observe that by swapping $1$ with $0$ (and $p_i$ with $1-p_i$) in certain directions, we can \emph{normalize} $S$ so that $\vec{0} \in S$ and $\D_S(\vec{0}) > 0$. In the following, we always assume, without loss of generality, that $S$ is normalized so that $\vec{0} \in S$ and $\D_S(\vec{0}) > 0$.

\paragraph{Notions of Distance between Distributions.} 
Let $\P, \Q$ be two probability measures in the discrete probability space $(\Omega, \mathcal{F})$. The \emph{total variation distance} between $\P$ and $\Q$, denoted $D_{TV}(\P,\Q)$, is defined as $D_{TV}(\P,\Q) = \frac{1}{2}\sum_{x \in \Omega}|\P(x) - \Q(x)| = \max_{A \in \mathcal{F}}|\P(A)-\Q(A)|$. The \emph{Kullback–Leibler divergence} (or simply, \emph{KL divergence}), denoted $D_{KL}(\P \parallel \Q)$, is defined as $D_{KL}(\P \parallel \Q) = \Exp_{x \sim \P}\!\left[\ln \frac{\P(x)}{\Q(x)}\right] = \sum_{x \in \Omega}\P(x) \ln \frac{\P(x)}{\Q(x)}$. We first recall that the KL divergence is additive for product distributions. 

\begin{proposition}\label{prop:kl-product}
Let $\mathcal{P}(\vec{p})$ and $\mathcal{Q}(\vec{q})$ be two Boolean product distributions. Then, 
\begin{equation}\label{eq:kl-product}
    D_{KL}(\mathcal{P} \parallel \mathcal{Q}) = \sum_{i=1}^d \left(p_{i}\ln\frac{p_{i}}{q_{i}} + (1-p_{i})\ln \frac{1-p_{i}}{1-q_{i}}\right)\,.
\end{equation}
\end{proposition}

Next, we observe that for two Bernoulli distributions, with parameters $p$ and $q$, the KL divergence can be upper bounded by the squared distance of their natural parameters. We provide the proof of \Cref{prop:kl-norm-ineq} in the \Cref{app:ineq}.

\begin{proposition}\label{prop:kl-norm-ineq}
For all $p, q \in (0,1)$, the following holds: 
\[    
 D_{KL}\big(\Be(p) \parallel \Be(q)\big) = p\ln \frac{p}{q} + (1-p)\ln\frac{1-p}{1-q} \leq 
  \left(\ln\frac{p}{1-p} - \ln\frac{q}{1-q}\right)^{2}\,.
\]
\end{proposition}

The following summarizes some standard upper bounds on the total variation distance and the KL divergence of two Boolean product distributions. 
\begin{proposition}\label{prop:kl-dtv-norm}
Let $\mathcal{P}(\vec{p})$ and $\mathcal{Q}(\vec{q})$ be two Boolean product distributions with $\vec p, \vec q \in (0,1)^d$, and let $\vec{z}(\vec{p})$ and $\vec{z}(\vec{q})$ be the vectors of their natural parameters. Then, the following hold: 
\begin{enumerate}
\item[$(i)$] $D_{KL}(\mathcal{P} \parallel \mathcal{Q}) \leq \| \vec{z}(\vec{p}) - \vec{z}(\vec{q}) \|_{2}^{2}$\,.

\item[$(ii)$] $D_{TV}(\mathcal{P} , \mathcal{Q}) \leq \frac{\sqrt{2}}{2}\,\| \vec{z}(\vec{p}) - \vec{z}(\vec{q}) \|_{2}$\,.

\item[$(iii)$] $D_{TV}(\mathcal{P} , \mathcal{Q}) \leq  \sqrt{ 2 \cdot \sum_{i=1}^d\frac{(p_i - q_i)^2}{(p_i+q_i)(2 - p_i - q_i)}}$\,.
\end{enumerate}
\end{proposition}

Now \Cref{prop:kl-dtv-norm} is an immediate consequence of \Cref{prop:kl-product}, \Cref{prop:kl-norm-ineq} and Pinsker's inequality (for (\textit{i}) and (\textit{ii})), and \cite[Lemma 2.17]{DKK+16b} (for (\textit{iii})).

\paragraph{Identifiability and Learnability.} 
A Boolean product distribution $\D(\vec{p})$ is \emph{identifiable} from its truncation $\D_S(\vec{p})$, if given $\D_S(\vec{p}; \vec{x})$, for all $\vec{x} \in S$, we can recover the parameter vector $\vec{p}$. 

A Boolean product distribution $\D(\vec{p})$ is \emph{efficiently learnable} from its truncation $\D_S(\vec{p})$, if for any $\eps, \delta > 0$, we can compute an estimation $\hat{\vec{p}}$ of the parameter vector $\vec{p}$ (or an estimation $\hat{\vec{z}}$ of the natural parameter vector $\vec{z}$) of $\D$ such that $\| \vec{p} - \hat{\vec{p}} \|_2 \leq \eps$ (or $\| \vec{z} - \hat{\vec{z}} \|_2 \leq \eps$), with probability at least $1-\delta$, with time and sample complexity polynomial in $d$, $1/\eps$ and $\ln(1/\delta)$ using truncated samples from $\D_S(\vec{p})$. By \Cref{prop:kl-dtv-norm}, an upper bound on the $L_2$ distance between $\vec{\hat{z}}$ and $\vec{z}$ (or between $\vec{\hat{p}}$ and $\vec{p}$) translates into an upper bound on the total variation distance between the true distribution and $\D(\hat{\vec{z}})$ (or $\D(\hat{\vec{p}})$). In this work, we identify sufficient and necessary conditions for efficient learnability of Boolean product distributions from truncated samples.

\section{Boolean Product Distributions Truncated by Fat Sets}
\label{s:fatsets}

In this section, we discuss \emph{fatness} of the truncation set, a strong sufficient (and in a certain sense, necessary) condition, under which we can generate samples from a Boolean product distribution $\D$ using samples from its truncation $\D_S$ (and access to $S$ through a membership oracle). 

\begin{definition}\label{def:fat}
A truncated Boolean product distribution $\D_S$ is $\alpha$-fat in coordinate $i \in [d]$, for some $\alpha > 0$, if $\Prob_{\vec{x} \sim \D_S}[ \flip(\vec{x},i) \in S] \geq \alpha$. A truncated Boolean product distribution $\D_S$ is $\alpha$-fat, for some $\alpha > 0$, if $\D_S$ is $\alpha$-fat in every coordinate $i \in [d]$. 
\end{definition}

If $\D_S$ is fat, it happens often that a sample $\vec{x} \sim \D_S$ has both $(\vec{x}_{-i}, 0), (\vec{x}_{-i}, 1) \in S$. Then,
conditional on the remaining coordinates $\vec{x}_{-i}$, the $i$-th coordinate $x_i$ of $\vec{x}$ is distributed as $\Be(p_i)$. We next focus on truncated Boolean product distributions $\D_S$ that are $\alpha$-fat. 

There are several natural classes of truncation subsets that give rise to fat truncated product distributions. E.g., for each $k \in [d]$, the halfspace $S_{\leq k} = \{ \vec{x} \in \Pi_d : x_1 + \ldots + x_d \leq k \}$ results in an $\alpha$-fat truncated distribution, if $\Prob_{\vec{x} \sim \D_{S_{\leq k}}}[x_i = 1] \geq \alpha$, for all $i \in [d]$. The same holds if $S$ is any \emph{downward closed}%
\footnote{A set $S \subseteq \Pi_d$ is downward closed if for any $\vec{x} \in S$ and any $\vec{y}$ with $y_i \leq x_i$, in all directions $i \in [d]$, $\vec{y} \in S$.}
subset of $\Pi_d$ and $\Prob_{\vec{x} \sim \D_{S}}[x_i = 1] \geq \alpha$, for all $i \in [d]$. 

Fatness in coordinate $i \in [d]$ is necessary, if we want to distinguish between two truncated Boolean distributions based on their $i$-th parameter only, if the remaining coordinates are correlated. Specifically, we can show that if $\D_S$ is $0$-fat in some coordinate $i$, there exists a Boolean distribution with $q_i \neq p_i$ (and $|q_i - p_i|$ large enough) whose truncation by $S$ appears identical to $\D_S$. Therefore, if the other coordinates are arbitrarily correlated, it is impossible to distinguish between the two distributions based on their $i$-th parameter alone. However, as we discuss in \Cref{s:conditions}, if $S$ is rich enough, but not necessarily fat, we can recover the entire parameter vector%
\footnote{\label{foot:linear}For a concrete example, where we can recover the entire parameter vector of a truncated Boolean product distribution $\D_S$, we consider $S = \{000, 110, 011, 101\} \subseteq\Pi_{3}$, which is not fat in any coordinate, and let $p_{\vec{x}} = \D_S(\vec{x})$, for each $\vec{x} \in S$. Then, setting $z_{i} = \ln\frac{p_{i}}{1-p_{i}}$, for each $i$, we can recover $(p_1, p_2, p_3)$, by solving the following linear system: $z_1 + z_2 = \ln\frac{p_{110}}{p_{000}}$, $z_2 + z_3 = \ln\frac{p_{011}}{p_{000}}$, $z_1 + z_3 = \ln\frac{p_{101}}{p_{000}}$\,. This is a special case of the more general identifiability condition discussed in \Cref{lem:identify}.}
of $\D$.

\begin{lemma}\label{lem:fat_necessary}
Let $i \in [d]$, let $S$ be any subset of $\Pi_d$ with $\flip(\vec{x}, i) \not\in S$, for all $\vec{x} \in S$, and consider any $0 < p < q < 1$. Then, for any Boolean distribution $\D_{-i}$ with $\D_{-i}(S_{-i}) \in (0, 1)$, there exists a distribution $\D'_{-i}$ such that $(\Be(p) \otimes \D_{-i})_S \equiv (\Be(q) \otimes \D'_{-i})_S$\,. 
\end{lemma}

\begin{proof}
We recall that $S_{-i} = \{ \vec{x}_{-i} : (\vec{x}_{-i}, 0) \in S \lor (\vec{x}_{-i}, 1) \in S \}$ denotes the projection of $S$ on $\Pi_{[d] \setminus \{ i \}}$. By hypothesis, $|S| = |S_{-i}|$ and for each $\vec{x}_{-i} \in S_{-i}$, either 
$(\vec{x}_{-i}, 0) \in S$ or $(\vec{x}_{-i}, 1) \in S$, but never both. For each $\vec{x}_{-i} \in S_{-i}$, we let:
\[
\D'_{-i}(\vec{x}_{-i}) = \left\{\begin{array}{ll} 
\D_{-i}(\vec{x}_{-i})\frac{p}{q} & 
\mbox{if $(\vec{x}_{-i}, 1) \in S$}\,, \\[2pt] 
\D_{-i}(\vec{x}_{-i}) \frac{1-p}{1-q} \ \ \ & 
\mbox{if $(\vec{x}_{-i}, 0) \in S$}\,.
\end{array}\right.
\]
For each $\vec{y} \in \Pi_{d-1} \setminus S_{-i}$, we let $\D'_{-i}(\vec{y}) \propto \D_{-i}(\vec{y})$, so that $\D'_{-i}$ is a probability distribution on $\Pi_{d-1}$\,. E.g., if for all $\vec{x}_{-i} \in S_{-i}$, $(\vec{x}_{-i}, 1) \in S$, we let 
\[ \D'_{-i}(\vec{y}) = \D_{-i}(\vec{y}) \frac{1-\D_{-i}(S_{-i}) \frac{p}{q} }{1-\D_{-i}(S_{-i})} \,. \]
By definition, $\Be(q) \otimes \D'_{-i}$ is a probability distribution on $\Pi_d$. Moreover, for all $\vec{x} \in S$, $(\Be(p) \otimes \D_{-i})(\vec{x}) = (\Be(q) \otimes \D'_{-i})(\vec{x})$, which implies the lemma. 

\end{proof}

\subsection{Sampling from a Boolean Product Distribution using Samples from its Fat Truncation}
\label{s:sample-fat}

An interesting consequence of fatness is that we can efficiently generate samples from a Boolean product distribution $\D$ using samples from any $\alpha$-fat truncation of $\D.$ The idea is described in \Cref{algo:fatsample}. \Cref{thm:sampler} shows that for any sample $\vec{x}$ drawn from $\D_S$ and any $i \in [d]$ such that $\flip(\vec{x}, i) \in S$, conditional on $\vec{x}_{-i}$, $x_i$ is distributed as $\Be(p_i)$. So, we can generate a random sample $\vec{y} \sim \D$ by putting together $d$ such values. $\alpha$-fatness of the truncated distribution $\D_S$ implies that the expected number of samples $\vec{x} \sim \D_S$ required to generate a $\vec{y} \sim \D$ is $O(\ln(d)/\alpha)$.  

\begin{algorithm}[ht]
\caption{Sampling from $\D$ using samples from $\D_{S}$}\label{algo:fatsample}
\begin{algorithmic}[1]
\Procedure{\textsc{Sampler}}{$\D_{S}$} \Comment{$\D_S$ \emph{is $\alpha$-fat.}}
	\State $\vec{y} \gets (-1, \ldots -1)$
	\While{$\exists y_{i} = -1$}
		\State Draw sample $\vec{x} \sim \D_{S}$
		\For{$i \gets 1, \ldots, d$}
    		\If{$\flip(\vec{x},i) \in S$} \Comment{\emph{We assume oracle access to} $S.$}
        		\State $y_{i} \gets x_{i}$
		    \EndIf
		\EndFor
	\EndWhile
	\State \textbf{return} $\vec{y}$ 
\EndProcedure
\end{algorithmic}
\end{algorithm}

\begin{theorem}\label{thm:sampler}
Let $\D$ be a Boolean product distribution over $\Pi_{d}$ and let $\D_S$ be any $\alpha$-fat truncation of $\D$. Then, (i) the distribution of the samples generated by \Cref{algo:fatsample} is identical to $\D$; and (ii) the expected number of samples from $\D_S$ before a sample is returned by \Cref{algo:fatsample} is $O(\ln(d)/\alpha)$. 
\end{theorem}

\begin{proof}
Let $\widetilde{\D}$ be the distribution of the samples generated by \Cref{algo:fatsample}. 
To prove that $\D$ and $\widetilde{\D}$ are identical, we show that $\widetilde{D}$ is a product distribution and that each $y_i \sim \Be(p_i)$, where $p_i$ is the parameter of $\D$ in direction $i \in [d].$ 

We fix a direction $i \in [d]$. Let $\D_{-i}$ denote the projection of $\D$ on 
$\Pi_{[d] \setminus \{ i \}}$. 
In \Cref{algo:fatsample}, $y_i$ takes the value of the $i$-coordinate of a sample $\vec{x} \sim \D_S$ such that both $(\vec{x}_{-i}, 0), (\vec{x}_{-i}, 1) \in S$. For each such sample $\vec{x}$, we have that:
\begin{equation} \label{eq:sample}
\D_S((\vec{x}_{-i}, 1)) = \frac{\D_{-i}(\vec{x}_{-i})\, p_i}{\D(S)}
 \ \ \ \ \mbox{and}\ \ \ \ \ 
 \D_S(\vec{x}_{-i}, 0) = \frac{\D_{-i}(\vec{x}_{-i})\, (1 - p_i)}{\D(S)}\,.
\end{equation} 
Therefore, $\frac{\D_S((\vec{x}_{-i}, 1))}{\D_S((\vec{x}_{-i}, 0))} = \frac{p_i}{1-p_i}$, which implies that $\D_S((\vec{x}_{-i}, 1)) = p_i$. Since this holds for all $\vec{x}_{-i}$ such that both $(\vec{x}_{-i}, 0), (\vec{x}_{-i}, 1) \in S$, $y_i$ is independent of the remaining coordinates $\vec{y}_{-i}$ and is distributed as $\Be(p_i)$. This concludes the proof of \emph{(i)}.

As for the sample complexity of \Cref{algo:fatsample}, we observe that since $\D_S$ is $\alpha$-fat in each coordinate $i$, each new sample $\vec{x}$ covers any fixed coordinate $y_i$ (i.e., $\vec{x}$ causes $y_i$ to become $x_i$) of $\vec{y}$ with probability at least $\alpha$. Therefore, the probability that any fixed coordinate $y_i$ remains $-1$ after \Cref{algo:fatsample} draws $k$ samples from $\D_S$ is at most $(1-\alpha)^k \leq e^{-\alpha k}$. Setting $k = 2\ln(d)/\alpha$ and applying the union bound, we get that the probability that there is a coordinate of $\vec{y}$ with value $-1$ after $2\ln(d)/\alpha$ samples from $\D_S$ is at most $d e^{-\alpha k} = d e^{-2\ln(d)} = 1/d$. Therefore, the expected number of samples from $\D_S$ before a random sample $\vec{y} \sim \D$ is returned by \Cref{algo:fatsample} is at most
\[ \frac{2\ln(d)}{\alpha} + \sum_{\ell=0}^\infty \frac{e^{-\ell \alpha}}{d} 
\leq \frac{2\ln(d)}{\alpha} + \frac{2}{d \alpha} = O\!\left(\frac{2\ln(d)}{\alpha}\right)\,, \]
where the inequality follows from $1/(1-e^{-\alpha}) \leq 2/\alpha$ for $\alpha \in (0,1).$ 

\end{proof}

\subsection{Parameter Estimation and Learning in Total Variation Distance} 

Based on \Cref{algo:fatsample}, we can recover the parameters of any Boolean product distribution $\D$ using samples from any fat truncation of $\D$. 
\
\begin{theorem}\label{thm:fat-estimation}
Let $\D(\vec{p})$ be a Boolean product distribution and let $\D_S(\vec{p})$ be a truncation of $\D$. If $\D_S$ is $\alpha$-fat in any fixed coordinate $i$, then, for any $\eps, \delta > 0$, we can compute an estimation $\hat{p}_i$ of the parameter $p_i$ of $\D$ such that $|p_i - \hat{p}_i| \leq \eps$, with probability at least $1-\delta$, using an expected number of $O(\ln(1/\delta)/(\eps^{2}\alpha))$ samples from $\D_S$. 
\end{theorem}

\begin{proof}
We modify \Cref{algo:fatsample} to \Cref{algo:fatsample_coordinate}, so that it generates random samples $y \in \{0,1\}$ in coordinate $i$ only. 
\begin{algorithm}[ht]
\caption{Sampling coordinate $i \in [d]$ from $\D$ using samples from $\D_{S}$}\label{algo:fatsample_coordinate}
\begin{algorithmic}[1]
\Procedure{\textsc{Sampler}}{$\D_{S}, i$} \Comment{$\D_S$ \emph{is fat in coordinate $i$.}}
	\State $y \gets -1$
	\While{$y = -1$}
		\State Draw sample $\vec{x} \sim \D_{S}$
    	\If{$\flip(\vec{x},i) \in S$} \Comment{\emph{We have oracle access to} $S.$}
        	\State $y \gets x_{i}$
		\EndIf
	\EndWhile
	\State \textbf{return} $y$ 
\EndProcedure
\end{algorithmic}
\end{algorithm}
As in \Cref{thm:sampler}.\textit{(i)}, each $y$ of \Cref{algo:fatsample_coordinate} is an independent sample from $\Be(p_i)$. Since the truncated distribution $\D_S$ is $\alpha$-fat, the expected number of samples from $\D_S,$ before $y$ is generated, is $1/\alpha$. We estimate $p_i$ from $n$ samples $y^{(1)}, \ldots, y^{(n)}$ of \Cref{algo:fatsample_coordinate} using the empirical mean $\hat{p}_i = \sum_{\ell=1}^n y^{(\ell)} / n$. A standard application of the Hoeffding bound%
\footnote{We use the following Hoeffding bound: Let $X_1, \ldots, X_n$ be $n$ independent Bernoulli random variables, let $X = \frac{1}{n}\!\left(\sum_{i=1}^n X_i\right)$ and $\Exp[X] = \frac{1}{n}\!\left(\sum_{i=1}^n \Exp[X_i]\right)$. Then, for any $t \geq 0$, $\Prob[|X - \Exp[X]| \geq t] \leq 2e^{-2nt^2}$.}
shows that if $n = \ln(2/\delta)/\eps^2$, then $|p_i - \hat{p}_i| \leq \eps$, with probability at least $1-\delta$. Hence, estimating $p_i$ with accuracy $\eps$ requires an expected number of $O(\ln(1/\delta)/(\eps^{2}\alpha))$ samples from $\D_S$.

\end{proof}


Using $n = \ln(2d/\delta)/\eps^2$ samples $\vec{y}^{(1)}, \ldots, \vec{y}^{(n)}$ generated by \Cref{algo:fatsample}, we can estimate all the parameters $\vec{p}$ of $\D$, by taking $\hat{p}_i = \sum_{\ell=1}^n y_i^{(\ell)} / n$, for each $i \in [d]$. The following is an immediate consequence of Theorems~\ref{thm:sampler}~and~\ref{thm:fat-estimation}.
 
\begin{corollary}\label{cor:fat-estimation}
Let $\D(\vec{p})$ be a Boolean product distribution and $\D_S(\vec{p})$ be any $\alpha$-fat truncation of $\D$. Then, for any $\eps, \delta > 0$, we can compute an estimation $\hat{\vec{p}}$ such that $\|\vec{p} - \hat{\vec{p}}\|_{\infty} \leq \eps$, with probability at least $1-\delta$, using an expected number of $O(\ln(d)\ln(d/\delta)/(\eps^{2}\alpha))$ samples from $\D_S$. 
\end{corollary}

\subsection{Identity and Closeness Testing with Access to Truncated Samples}
\label{s:fat-corollaries}

\Cref{thm:sampler} implies that if we have sample access to an $\alpha$-fat truncation $\D_S$ of a Boolean product distribution $\D$, we can pretend that we have sample access to the original distribution $\D$, at the expense of an increase in the sample complexity (from $\D_S$) by a factor of $O(\ln(d)/\alpha)$. Therefore, we can extend virtually all known hypothesis testing and learning algorithms for Boolean product distributions to fat truncated Boolean product distributions. 

For \emph{identity testing} of Boolean product distributions, based on samples from fat truncated ones, we combine \Cref{algo:fatsample} with the algorithm of \cite[Sec.~4.1]{CDKS17}. Combining \Cref{thm:sampler} with \cite[Theorem~6]{CDKS17}, we obtain the following:

\begin{corollary}[Identity Testing]\label{cor:fat-identity}
Let $\mathcal{Q}(\vec{q})$ be a Boolean product distribution described by its parameters $\vec{q}$, and let $\D$ be a Boolean product distribution for which we have sample access to its $\alpha$-fat truncation $\D_S$. For any $\eps > 0$, we can distinguish between $D_{TV}(\mathcal{Q}, \D) = 0$ and $D_{TV}(\mathcal{Q}, \D) > \eps$, with probability $2/3$, using an expected number of  $O(\ln(d)\sqrt{d}/(\alpha\eps^2))$ samples from $\D_S$.
\end{corollary}

We can extend \Cref{cor:fat-identity} to \emph{closeness testing} of two Boolean product distributions, for which we only have sample access to their fat truncations. We combine \Cref{algo:fatsample} with the algorithm of \cite[Sec.~5.1]{CDKS17}. The following is an immediate consequence of \Cref{thm:sampler} and \cite[Theorem~9]{CDKS17}.   

\begin{corollary}[Closeness Testing]\label{cor:fat-closeness}
Let $\mathcal{Q}$, $\D$ be two Boolean product distributions for which we have sample access to their $\alpha_1$-fat truncation $\mathcal{Q}_{S_1}$ and $\alpha_2$-fat truncation $\D_{S_2}$. For any $\eps > 0$, we can distinguish between $D_{TV}(\mathcal{Q}, \D) = 0$ and $D_{TV}(\mathcal{Q}, \D) > \eps$, with probability at least $2/3$, using an expected number of $O\!\left((\frac{\ln(d)}{\alpha_1}+\frac{\ln(d)}{\alpha_2})\max\{\sqrt{d}/\eps^2, d^{3/4}/\eps\}\right)$ samples from $\mathcal{Q}_{S_1}$ and $\D_{S_2}$.
\end{corollary}

\subsection{Learning in Total Variation Distance} 

Using \Cref{algo:fatsample}, we can learn a Boolean product distribution $\D(\vec{p})$, within $\eps$ in total variation distance, using samples from its fat truncation. The following uses a standard analysis of the sample complexity of learning a Boolean product distribution (see e.g., \cite{kamath2018privately}). 

\begin{corollary}\label{cor:fat-learning}
Let $\D(\vec{p})$ be a Boolean product distribution and let $\D_S$ be any $\alpha$-fat truncation of $\D$. Then, for any $\eps, \delta > 0$, we can compute a Boolean product distribution $\hat{\D}(\hat{\vec{p}})$ such that $D_{TV}(\D, \hat{\D}) \leq \eps$, with probability at least $1-\delta$, using $O(d \log(d) \ln(d/\delta)/(\eps^{2}\alpha))$ samples from $\D_S$. 
\end{corollary}

\begin{proof}
We assume that $p_i \leq 1/2$ and that for all $i \in [d]$, $p_i \geq \eps/(8d)$. Both are without loss of generality. The former can be enforced by flipping $0$ and $1$. For the latter, we observe that there exists a distribution $\D'$ with $D_{TV}(\D, \D') \leq \eps/2$ that satisfies the assumption ($\D'$ can be obtained from $\D$ by adding uniform noise in each coordinate with probability $1-\frac{\eps}{4d}$, see also \cite[Sec.~4.1]{CDKS17}).

By \Cref{prop:kl-dtv-norm}, for any two Boolean product distributions $\D(\vec{p})$ and $\hat{\D}(\hat{\vec p})$ with parameter vectors $\vec p, \hat{\vec p} \in (0,1)^d$, it holds that  
\begin{equation}\label{eq:dtv}
D_{TV}(\D, \hat{\D}) \leq \sqrt{ 2 \cdot \sum_{i=1}^d \frac{(p_i - \hat{p}_i)^2}{(p_i+\hat{p}_i)(2 - p_i - \hat{p}_i)}}\,.
\end{equation}

Similarly to the proof of \Cref{cor:fat-estimation}, we take $n$ samples $\vec{y}^{(1)}, \ldots, \vec{y}^{(n)}$ from \Cref{algo:fatsample} and estimate each parameter $p_i$ of $\D$ as $\hat{p}_i = \sum_{\ell=1}^n y_i^{(\ell)} / n$. Using the Chernoff bound in \cite[Claim~5.16]{kamath2018privately}, we show that for all directions $i \in [d]$, $\frac{(p_i - \hat{p}_i)^2}{(p_i+\hat{p}_i)(2 - p_i - \hat{p}_i)} \leq O(\ln(d/\delta)/n)$. Drawing $n = O(d\ln(d/\delta)/\eps^2)$ samples from \Cref{algo:fatsample} and using \Cref{eq:dtv}, we get that $D_{TV}(\D, \hat{\D}) \leq O(\eps)$. The sample complexity follows from the fact that each sample of \Cref{algo:fatsample} requires an expected number of $O(\ln(d)/\alpha)$ samples from the $\alpha$-fat truncation $\D_S$ of $\D$. 

\end{proof}

We can improve the sample complexity in \Cref{cor:fat-learning}, if the original distribution $\D$ is sparse. We say that a Boolean product distribution $\D(\vec{p})$ is \emph{$(k, c)$-sparse}, for some $k \in [d]$ and $c \in [0,1]$, if there is an index set $I \subseteq[d]$, with $|I| = d -k$, such that for all $i \in I$, $p_i = c$. Namely, we know that $d-k$ of $\D$'s parameters are equal to $c$ (but we do not know which of them). Then, we first apply \Cref{cor:fat-estimation} and estimate all parameters of $\D$ within distance $\eps/\sqrt{k}$. We set each $p_i$ with $|p_i - c| \leq \eps/\sqrt{k}$ to $p_i = c$. Thus, we recover the index set $I$. For the remaining $k$ parameters, we apply \Cref{cor:fat-learning}. The result is summarized by the following: 

\begin{corollary}\label{cor:fat-learning-sparse}
Let $\D(\vec{p})$ be a $(k, c)$-sparse Boolean product distribution and let $\D_S$ be any $\alpha$-fat truncation of $\D$. Then, for any $\eps, \delta > 0$, we can compute a Boolean product distribution $\hat{\D}(\vec{\hat{p}})$ such that $D_{TV}(\D, \hat{\D}) \leq \eps$, with probability at least $1-\delta$, using $O\!\left(\frac{k\ln(d)\ln(d/\delta)}{\eps^2\alpha}\right)$ samples from the truncated distribution $\D_S$. 
\end{corollary}

\subsection{Learning Ranking Distributions from Truncated Samples} 
\label{s:ranking}

An interesting application of \Cref{thm:sampler} is parameter estimation of ranking distributions from truncated samples. For clarity, we next focus on Mallows distributions. Our techniques imply similar results for other well known models of ranking distributions, such as Generalized Mallows distributions~\cite{FlignerV1986} and the models of \cite{Plackett,Luce}, \cite{BradleyTerry} and \cite{Babington}.

\paragraph{Definition and Notation.}
We start with some notation specific to this section. Let $\S_{d}$ be the symmetric group over the finite set of items $[d]$. Given a ranking $\pi \in \S_d$, we let $\pi(i)$ denote the position of item $i$ in $\pi$. We say that $i$ precedes $j$ in $\pi$, denoted by $i \succ_{\pi} j$, if $\pi(i) < \pi(j)$. The Kendall tau distance of two rankings $\pi$ and $\sigma$, denoted by $D_{\tau}(\pi, \sigma)$, is the number of discordant item pairs in $\pi$ and $\sigma$. Formally, 
\begin{equation}
    D_{\tau}(\pi, \sigma) = \sum_{1 \leq i < j \leq d} \vec{1}\{(\pi(i) - \pi(j))(\sigma(i) - \sigma(j)) < 0 \}\,.
\end{equation}

The \emph{Mallows model} \cite{mallows1957non} is a family of ranking distributions parameterized by the \emph{central ranking} $\pi_0 \in \S_d$ and the \emph{spread parameter} $\phi \in [0,1]$. Assuming the Kendall tau distance between rankings, the probability mass function is $\M(\pi_0, \phi; \pi) = \phi^{D_\tau(\pi_0, \pi)}/Z(\phi)$, where the normalization factor is $Z(\phi) = \prod_{i=1}^{d} \frac{1-\phi^{i}}{1-\phi}$. 
For a given Mallows distribution $\M(\pi_0, \phi)$, we denote $p_{ij} = \Prob_{\pi \sim \M}[ i \succ_{\pi} j]$ the probability that item $i$ precedes item $j$ in a random sample from $\M$. 

\paragraph{Truncated Mallows Distributions.}
We consider parameter estimation for a Mallows distribution $\M(\pi_0, \phi)$ with sample access to its truncation $\M_S$ by a subset $S \subseteq\S_d$. Then, $\M_S(\pi) = \M(\pi) / \M(S)$, for each $\pi \in S$, and $\M_S(\pi) = 0$, otherwise. Next, we generalize the notion of fatness to truncated ranking distributions and prove the equivalent of \Cref{thm:fat-estimation} and \Cref{cor:fat-estimation}.

For a ranking $\pi$, we let $\flip(\pi, i, j)$ denote the ranking $\pi'$ obtained from $\pi$ with the items $i$ and $j$ swapped. Formally, $\pi'(\ell) = \pi(\ell)$, for all items $\ell \in [d] \setminus \{i, j\}$, $\pi'(j) = \pi(i)$ and $\pi'(i) = \pi(j)$. We say that a truncated Mallows distribution $\M_S$ is \emph{$\alpha$-fat for the pair $(i, j)$}, if $\Prob_{\pi \sim \M_S}[ \flip(\pi,i,j) \in S] \geq \alpha$, for some $\alpha > 0$. A truncated Mallows distribution $\M_S(\pi_0, \phi)$ is $\alpha$-fat, if $\M_S$ is $\alpha$-fat for all pairs $(i, j)$, and \emph{neighboring $\alpha$-fat}, if $\M_S$ is $\alpha$-fat for all pairs $(i, j)$ that occupy neighboring positions in the central ranking $\pi_0$, i.e., for all pairs $(i, j)$ with $|\pi_0(i) - \pi_0(j)| = 1$. 

\paragraph{Parameter Estimation and Learning of Mallows Distributions from Truncated Samples.}
We present \Cref{algo:ranking} that draws a sample from the truncated Mallows distribution $\M_S$ and updates a vector $\vec{q}$ with estimations $\hat{p}_{ij} = q_{ij}/(q_{ij}+q_{ji})$ of the probability $p_{ij}$ that item $i$ precedes item $j$ in a sample from the true Mallows distribution $\M$.
\begin{algorithm}[ht]
\caption{Update the estimate $q_{ij}$ using one sample from $\M_S$}\label{algo:ranking}
\begin{algorithmic}[1]
\Procedure{\textsc{Sample}}{$\M_{S}, \vec{q}$} \Comment{$\M_S$ \emph{is (neighboring) $\alpha$-fat.}}
\State Draw sample $\pi \sim \M_{S}$
\For{all $(i,j)$ such that $\flip(\pi, i, j) \in S$} \Comment{\emph{We assume oracle access to} $\M_S.$}
\If{$i \succ_{\pi} j$} 
\State $q_{ij} \gets q_{ij} + 1$
\Else 
\State $q_{ji} \gets q_{ji} + 1$
\EndIf
\EndFor
\State \textbf{return} $\vec{q}$
\EndProcedure
\end{algorithmic}
\end{algorithm}

The vector $\vec{q}$ is initialized to $0$ for all item pairs $(i, j)$ and is updated through successive calls to \Cref{algo:ranking}. For each sample $\pi \sim \M_S$, \Cref{algo:ranking} updates either $q_{ij}$ or $q_{ji}$ for all item pairs $(i, j)$ such that $\flip(\pi, i, j) \in S$. 
Thus, we can show the following: 

\begin{theorem}\label{thm:mallows-fat}
Let $\M(\pi_0, \phi)$ be a Mallows distribution with $\pi_0 \in \S_d$ and $\phi \in [0, 1-\gamma]$, for some constant $\gamma > 0$, and let $\M_S$ be any neighboring $\alpha$-fat truncation of $\M$. Then, 
\begin{enumerate}
\item[(i)] For any $\delta > 0$, we can learn the central ranking $\pi_0$, with probability at least $1-\delta$, using an expected number of $O(\ln(d)\ln(d/\delta)/(\gamma^2\alpha))$ samples from $\M_S$. 

\item[(ii)] Assuming that the central ranking $\pi_0$ is known, for any $\eps, \delta > 0$, we can compute an estimation $\hat{\phi}$ of the spread parameter such that $|\phi - \hat{\phi}| \leq O(\eps)$, with probability at least $1-\delta$, using an expected number of $O(\ln(1/\delta)/(\eps^{2}\alpha))$ samples from $\M_S$. 

\item[(iii)] For any $\eps,\delta > 0$, we can compute a Mallows distribution $\hat{\M}(\pi_0, \hat{\phi})$ so that
$$
    D_{TV}(\M, \hat{\M}) \leq O(\eps)\,,
$$
with probability at least $1-\delta,$ using an expected number of
$$
    O(\ln(d)\ln(d/\delta)/(\gamma^2\alpha)+d\ln(1/\delta)/(\eps^{2}\alpha))
$$ 
samples from $\M_S$. 
\end{enumerate}
\end{theorem}

The following is similar in spirit to \Cref{thm:fat-estimation}. To estimate $p_{ij}$, we call \Cref{algo:ranking} as long as $q_{ij}+q_{ji} < \ln(2/\delta)/\eps^2$. For the proof, we apply the argument used in the proof of \Cref{thm:sampler}.$(i)$ and the Hoeffding bound used in the proof of \Cref{thm:fat-estimation}. 
\newline
\begin{corollary}\label{cor:randking-sampler}
Let $\M$ be a Mallows distribution and let $\M_S$ be any truncation of $\M$. If $\M_S$ is $\alpha$-fat for pair $(i, j)$, for any $\eps, \delta > 0$, we can compute an estimation $\hat{p}_{ij}$ of the probability $p_{ij} = \Prob_{\pi \sim \M}[ i \succ_{\pi} j]$ such that $|p_{ij} - \hat{p}_{ij}| \leq \eps$, with probability at least $1-\delta$, using an expected number of $O(\ln(1/\delta)/(\eps^{2}\alpha))$ samples from $\M_S$. 
\end{corollary}

We next give a detailed proof of \Cref{thm:mallows-fat}, which shows how \Cref{algo:ranking} can efficiently estimate the parameters of (and learn in total variation distance) a Mallows distribution $\M$ using samples from any neighboring $\alpha$-fat truncation $\M_S$ of $\M$. 

\begin{proof}[Proof of \Cref{thm:mallows-fat}] \label{proof:mallthm}
To prove (\textit{i}), we use the fact that there is a bijective mapping from rankings in $\S_d$ to transitive tournaments on $d$ nodes. So, we think of $\vec{q}$ as a directed graph $G$ on $d$ nodes, where there is an edge between $i$ and $j$ if $q_{ij}+q_{ji} \geq n$, for some $n$ sufficiently large, which, for simplicity, will be determined at the end of the proof. The edge is from $i$ to $j$, if $q_{ij} > q_{ji}$, and from $j$ to $i$, otherwise. We keep calling \Cref{algo:ranking} until a directed path including all nodes (i.e., a total order) is formed in $G$. If a cycle is formed in $G$, before a total order appears, we discard $\vec{q}$ and start the algorithm from scratch.

Since $\M_S$ is neighboring $\alpha$-fat, for any such pair $(i,j)$ of neighboring items in $\pi_0$, the probability that a fresh sample $\pi \sim \M_S$ in \Cref{algo:ranking} increases $q_{ij}+q_{ji}$ is at least $\alpha$ (by the definition of neighboring $\alpha$-fatness). Using exactly the same reasoning as in the proof of \Cref{thm:sampler}.$(ii)$, we show that the expected number of samples before $d$ edges appear in $G$ is $O(n \ln(d)/ \alpha)$. 

Let us fix any pair of items $i$ and $j$ such that $i \succ_{\pi_0} j$ and there is an edge between $i$ and $j$ in $G$. For simplicity, we assume that $q_{ij}+q_{ji} = n$. For sake of intuition, one may think of $i$ and $j$ as neighboring in $\pi_0$, but our analysis does not require so. We note that $\Exp[q_{ij}] = n p_{ij}$ and $\Exp[q_{ji} ] = n p_{ji}$, and let $m_{ij} = p_{ij} - p_{ji}$. Working as in \cite[(1)]{caragiannis2013noisy}, we can show that $m_{ij} \geq  \frac{1+\phi}{1-\phi} = \Omega(\gamma)$ (see also \cite[Theorem~12]{busa2019optimal}). Therefore, $\Exp[q_{ij}] = n \cdot \frac{1+ m_{ij}}{2}$ and $\Exp[q_{ji}] = n \cdot \frac{1 - m_{ij}}{2}$. A standard application of the Hoeffding bound shows that if $n = O(\ln(d/\delta)/m_{ij}^2)$, $\Prob[ q_{ij} \leq n/2] \leq \delta/d^2$. Therefore, assuming that an edge between $i$ and $j$ is present in $G$, the edge is directed from $i$ to $j$ (i.e., as in $\pi_0$) with probability at least $1-\delta/d^2$. Applying the union bound, we get that when we stop calling \Cref{algo:ranking}, all edges present in $G$ are as in $\pi_0$ with probability at least $1-\delta$. 

We are ready to finish the proof of Item $(i)$. Putting everything together, we get that after an expected number of $O(\ln(d)\ln(d/\delta)/(\alpha\gamma^2))$ samples from the truncated Mallows distribution $\M_S$, a total order consistent with $\pi_0$ is formed in $G$, with probability at least $1-\delta$. Increasing $n$ by a constant factor makes the probability that a cycle appears in $G$ polynomially small in $d$, which allows us to bound the expected number of samples from $\M_S$ before we find a total order in $G$ by $O(\ln(d)\ln(d/\delta)/(\alpha\gamma^2))$.

For (\textit{ii}), we assume that we know the central ranking $\pi_0$. For simplicity, we assume that $\pi_0 = (1, \ldots, d)$. Then, as in \Cref{cor:randking-sampler}, we can estimate the probability $p_{12} = \Prob_{\pi \sim \M}[ 1 \succ_{\pi} 2]$ such that $|p_{12} - \hat{p}_{12}| \leq \eps$, with probability at least $1-\delta$, using an expected number of $O(\ln(1/\delta)/(\eps^{2}\alpha))$ samples from $\M_S$. Using $\hat{p}_{12}$, we compute an estimation $\hat{m}_{12} = 2\hat{p_{12}} - 1$ of $m_{12} = 2p_{12}-1$. It is straightforward to verify that $|p_{12} - \hat{p}_{12}| \leq \eps$ implies that $|m_{12} - \hat{m}_{12}| \leq \eps$. Working as in \cite[(1)]{caragiannis2013noisy}, we show that for each pair of neighboring items $i$ and $i+1$ in the central ranking $\pi_0$, $m_{i(i+1)} = \frac{1-\phi}{1+\phi}$. The reason is that for any ranking $\pi$ and any pair of items $i$ and $i+1$, with $i \succ_{\pi} i+1$, that are neighboring in $\pi_0$, swapping $i$ and $i+1$ results in a ranking $\pi'$ with $D_{\tau}(\pi', \pi_0) = D_{\tau}(\pi, \pi_0)+1$. Our estimation of $\phi$ is $\hat{\phi} = \frac{1-\hat{m}_{12}}{1+\hat{m}_{12}}$, where $|m_{12} - \hat{m}_{12}| \leq \eps$ implies that $|\phi - \hat{\phi}| \leq O(\eps)$. 

Part (\textit{iii}) follows from (\textit{i}), (\textit{ii}) and \cite[Theorem~15]{busa2019optimal}. We can learn $\pi_0$ using the algorithm of (\textit{i}) and an estimation $\hat{\phi}$ of $\phi$ such that $|\hat{\phi} - \phi| \leq \eps/\sqrt{d}$ using the estimator of (\textit{ii}), with an expected number of $O(d\ln(1/\delta)/(\eps^{2}\alpha))$ samples from $\M_S$. \cite[Theorem~15]{busa2019optimal} shows that if $|\hat{\phi} - \phi| \leq \eps/\sqrt{d}$, then $D_{TV}(\M(\pi_0, \phi), \hat{\M}(\pi_0, \hat{\phi}) \leq O(\eps)$. 

\end{proof}

In this section, we focused on various implications of a truncation set being fat. We close this section with some comments about efficiently learning truncated Mallows models and performing e.g., identity testing when the $\alpha$-fatness property does not hold true.

Let us recall the problem of learning truncated Mallows models. We will focus on estimating the central ranking assuming that the dispersion parameter is known. In this setting, there exists a central ranking $\pi_0$ and the learner observes i.i.d. samples from $\M_S(\pi_0, \phi)$. The goal is to efficiently estimate $\pi_0$.
In the non-truncated setting, $\Theta(\log(d))$ samples are required. Under the fatness condition, we provided an $O(\log^2(d))$ sample algorithm. However, the fatness condition can be dropped but it may be still possible to retrieve the central ranking. Using the techniques of the upcoming sections, one could execute the Projected SGD approach (Section~\ref{s:sgd}) and, under some structural conditions on the Boolean product distribution of dimension $O(d^2)$ and the truncation set (e.g., anti-concentration), recover the central ranking using $\poly(d)$ samples. However, it is not clear whether this reduction is optimal. It is an interesting question for future work to give the right characterization of learnability for truncated Mallows distributions.

In the task of identity testing of truncated Boolean product distributions, there exists a target distribution $\D^\star$ specified to the tester via its $d$ success probabilities and the algorithm observes i.i.d. samples from the unknown truncated Boolean product distribution $\D_S$. The goal is to accept if $\D = \D^\star$ with probability $2/3$ and to reject if $D_{TV}(\D,\D^\star) >  \eps$ with probability $2/3$. We assume that the tester has
membership oracle access to the set $S$ (note that the truncated target $\D^\star_S$ cannot even be parsed efficiently by the tester since its size may be exponential in $d)$. If the fatness condition fails but the conditions of Section \ref{s:conditions} hold true, then one could still perform the SGD approach (Section~\ref{s:sgd}), learn the distribution and hence perform identity testing using a polynomial number of samples. It is an interesting question whether one could efficiently perform identity testing from truncated samples without learning the distribution.

\section{Efficient Learnability from Truncated Samples: Necessary Conditions}
\label{s:conditions}

We next discuss necessary conditions for identifiability and efficient learnability of a Boolean product distribution from truncated samples. For \Cref{as:identify} and \Cref{lem:identify}, we recall that we can assume without loss of generality that $S$ is normalized so that $\D_S(\vec{0}) > 0$. 

\begin{assumption}\label{as:identify}
For the truncated Boolean product distribution $\D_S$, $\D_S(\vec{0}) > 0$ (after possible normalization) and there are $d$ linearly independent $\vec{x}^{(1)}, \ldots, \vec{x}^{(d)} \in S$ with $\D_S(\vec{x}^{(j)}) > 0$, $j \in [d]$. 
\end{assumption}

The proof of \Cref{lem:identify} demonstrates that recovering $\vec{p}$ requires the solution to a linear system, similar to that in \Cref{foot:linear}, which is solvable if and only if \Cref{as:identify} holds. 

\begin{lemma}\label{lem:identify}
A Boolean product distribution $\D(\vec{p})$ on $\Pi_d$ is identifiable from its truncation $\D_S$ if and only if \Cref{as:identify} holds. 
\end{lemma}

\begin{proof}
Let us assume that $\vec{0} \in S$ and there are $d$ linearly independent vectors $\vec{x}^{(1)}, \ldots, \vec{x}^{(d)} \in S$. We have that $\D(\vec{0}) = \prod_{i=1}^d (1-p_i)$, and for each $j \in [d]$, 
\begin{equation}\label{eq:prob_i}
 \prod_{i: x_i^{(j)} = 1} p_i  \prod_{i: x_i^{(j)} = 0} (1-p_i) = \D(\vec{x}^{(j)}) \,. 
\end{equation}
However, the right-hand side of \Cref{eq:prob_i} cannot be directly obtained from the truncated distribution $\D_S$. Hence, we normalize \Cref{eq:prob_i}, by dividing both sides by $\D_S(\vec{0})$, and get that
\begin{equation}\label{eq:prob_i_2}
 \prod_{i: x_i^{(j)} = 1} \frac{p_i}{1-p_i} = \frac{\D(\vec{x}^{(j)})}{\D(\vec{0})} \,. 
\end{equation}
We observe that $\frac{\D(\vec{x}^{(j)})}{\D(\vec{0})} = \frac{\D_S(\vec{x}^{(j)})}{\D_S(\vec{0})}$, because for all $\vec{x} \in S$, $\D_S(\vec{x}) = \D(\vec{x})/\D(S)$. So, after normalization, the right-hand side of \Cref{eq:prob_i_2} becomes a constant $q_j \eqdef \frac{\D_S(\vec{x}^{(j)})}{\D_S(\vec{0})} > 0$, for all $j \in [d]$. 

Taking logarithms in \Cref{eq:prob_i_2}, we obtain that $\sum_{i: x_i^{(j)} = 1} z_i = \ln q_j$, where $z_i = \ln\frac{p_i}{1-p_i}$, or equivalently $\vec{z}^T \vec{x}^{(j)} = \ln q_j$. Since $\vec{x}^{(1)}, \ldots, \vec{x}^{(d)}$ are linearly independent, the corresponding linear system with $d$ equations and $d$ unknowns has a unique solution. Solving the linear system $\left\{  \vec{z}^T {\vec{x}^{(j)}} = \ln q_j \right\}_{j \in [d]}$, we recover $\vec{z}$ and eventually $\vec{p}$.

The converse follows from the observation that solving a linear system as the one above is the only way to recover $\vec{p}$ from $\D_S$ (a linear system is the input to any potential solver from an information-theoretic viewpoint). Specifically, the only way to recover $\vec{p}$ from $\D_S$ is to solve the system consisting of \Cref{eq:prob_i}, for $j = 1, \ldots, d$, or some other equivalent system with $d$ equations and $p_1, \ldots, p_d$ as unknowns. The only way to recover $\D(\vec{x})$ is to normalize \Cref{eq:prob_i} by dividing by $\D(\vec{x}')$, for some $\vec{x}' \in S$ with $\D_S(\vec{x}') > 0$. We can assume without loss of generality that $\vec{x}' = \vec{0}$, since we can normalize $S$ so that $\vec{x}'$ becomes $\vec{0}$. After normalizing by $\D_S(\vec{0})$ and taking logarithms in \Cref{eq:prob_i_2}, recovering $\vec{z}$ and $\vec{p}$ requires a collection of $d$ linearly independent equations, which correspond to $d$ linearly independent  $\vec{x}^{(1)}, \ldots, \vec{x}^{(d)} \in S$ with $\D_S(\vec{x}^{(j)}) > 0$, for each $j \in [d]$. Technically, if \Cref{as:identify} does not hold, the input contains a matrix with rank $< d$ and hence the true $\vec p$ is not uniquely identifiable.

\end{proof}

We proceed to show two necessary conditions for \emph{efficient learnability}. Our first condition is that we have oracle access to the truncation set $S$. More formally, we assume that: 

\begin{assumption}\label{as:oracle}
$S$ is accessible through a membership oracle, which reveals whether $\vec{x} \in S$, for any $\vec{x} \in \Pi_d$. 
\end{assumption}

Based on the proof of \cite[Lemma~12]{DGTZ18}, we show that if \Cref{as:oracle} does not hold, we can construct a (possibly exponentially large) truncation set $S$ so that $\D_S$ appears identical to the uniform distribution $\U$ on $\Pi_d$ as long as all the samples are distinct. 

\begin{lemma}\label{lem:oracle}
For any Boolean product distribution $\D(\vec{p})$, there is a truncation set $S$ so that without additional information about $S$, we cannot distinguish between sampling from $\D_S$ and sampling from the uniform distribution $\U$ on $\Pi_d$, before an expected number of $\Omega(\sqrt{|S|})$ samples are drawn. 
\end{lemma}

\begin{proof}
The truncation set $S = S_1 \times \cdots \times S_d$ is the product of $d$ truncation sets $S_i$, one in each direction $i \in [d]$. If $p_i \geq 1/2$, $S_i = \{ 0, 1 \}$ with probability $\frac{1-p_i}{p_i}$, and $S_i = \{ 0 \}$, otherwise. If $p_i < 1/2$, $S_i = \{ 0, 1 \}$ with probability $\frac{p_i}{1-p_i}$, and $S_i = \{ 1 \}$, otherwise. There is a constant $c > 0$ such that if $|p_i - 1/2| \leq c$, for all $i \in [d]$, $|S|$ is exponential in $d$ with constant probability. 

By the principle of deferred decisions, we can think of the sampling process from $\D_S$ as follows: we draw a sample $\vec{x} \sim \D$. If this is the first time that $\vec{x}$ is drawn from $\D$, for each $i\in [d]$, independently, $x_i$ survives with probability $\min\{ \Be(p_i; 1-x_i)/\Be(p_i; x_i), 1 \}$. If every $x_i$ survives, $\vec{x}$ is added to $S$ and becomes a sample from $\D_S$. If $\vec{x}$ has been drawn before, $\vec{x}$ becomes a sample from $\D_S$ if and only if $\vec{x} \in S$, so that new samples are treated consistently with past ones. 

We note that as long as a duplicate sample does not appear, the probability that $x_i = 0$ and $x_i$ survives is equal to the probability that $x_i = 1$ and $x_i$ survives, for all $i \in [d]$. In fact, the following process samples from the uniform distribution $\U_{d}$ on $\Pi_d$: we draw a sample $\vec{x} \sim \D$. Then, for each $i \in [d]$, independently, $x_i$ survives with probability $\min\{ \Be(p_i; 1-x_i)/\Be(p_i; x_i), 1 \}$. If every $x_i$ survives, $\vec{x}$ is returned as a sample from $\U_{d}$. The difference is that there is no truncation set. So, we do not need to treat new samples consistently with past ones.  

Before the first duplicate sample is drawn from $\D_S$, there is no way to distinguish between sampling from $\D_S$ and sampling from $\U_{d}$. By the birthday problem, the appearance of the first duplicate sample from $\D_S$ requires an expected number of $\Omega(\sqrt{|S|})$ samples from $\D_S$.

We highlight that we can easily distinguish between sampling from $\D_S$ and sampling from $\U$, if we have oracle access to the truncation set $S$. 

\end{proof}

Our second necessary condition for efficient learnability is that the truncated distribution is not extremely well concentrated in any direction. Intuitively, we need the Boolean product distribution $\D$, and its truncation $\D_S$, to behave well, so that we can get enough information about $\D$ based on few samples from $\D_S$. More formally, we quantify $\D_S$'s anti-concentration using $\lambda^\star$, which is the maximum positive number so that for all unit vectors $\vec{w} \in \reals^d$,  $\|\vec{w} \|_2 = 1$, and all $c \in \reals$, $\Prob_{\vec{x} \sim \D_S}[ \vec{w}^T \vec{x} \not\in ( c-\lambda^\star, c+\lambda^\star)] \geq \lambda^\star$. 
{
\begin{assumption}\label{as:anticon}
There exists a $\lambda > 0$ such that for all unit vectors $\vec{w} \in \reals^d$,  $\|\vec{w} \|_2 = 1$, and all $c \in \reals$, $\Prob_{\vec{x} \sim \D_S}[ \vec{w}^T \vec{x} \not\in ( c-\lambda , c+\lambda )] \geq \lambda$.
\end{assumption}
}
We note that \Cref{as:anticon} is a stronger version of \Cref{as:identify}. It also implies that all parameters $p_i \in (0, 1)$ are bounded away from $0$ and $1$ by a safe margin.

{We now discuss the scaling of the parameter $\lambda$ as a function of the dimension $d$. First, if $\lambda$ is very small, i.e., it is smaller than any polynomial of $1/d$, then we show that a super-polynomial dependence on $d$ is unavoidable in the sample complexity. To do that, }
we show that if $\D_S$ is well concentrated in some direction, estimating the parameter vector $\vec{p}$ requires a large number of samples from $\D_S$. More specifically, we show that either estimating $\D_S(\vec{0})$, which is needed for normalizing the linear system in \Cref{lem:identify}, or sampling $d$ vectors that result in a well-conditioned linear system, require $\Omega(1/\lambda^\star)$ samples from $\D_S$. Therefore, if \Cref{as:anticon} does not hold, estimating $\vec{p}$ with truncated samples from $\D_S$ has superpolynomial sample complexity. 

\begin{lemma}\label{lem:anticon}
{
Let $\D(\vec{p})$ be a Boolean product distribution and let $\D_S$ be a truncation of $\D$.
Assume that the anti-concentration parameter $\lambda^\star$ satisfies $1/\lambda^\star = \omega(\poly(d))$.
Then, computing an estimation $\vec{\hat{p}}$ of the parameter vector $\vec{p}$ of $\D$ such that $\| \vec{p} - \hat{\vec{p}} \|_2 \leq o(1)$ requires an expected number of $\Omega(1/\lambda^\star)$ samples from $\D_S$. }
\end{lemma}
Let us first provide some intuition. For a unit vector $\vec{w} \in \reals^d$, we think of the space $H_{\vec{w}} = \{ \vec{x} \in S : \vec{w}^T \vec{x} \in ( c-\lambda, c+\lambda ) \}$. If $\lambda^\star$ is very small, there is a direction $\vec{w}$ such that virtually all samples $\vec{x} \sim \D_S$ lie in $H_{\vec{w}}$. Intuitively, recovering ($\vec{z}$ and) $\vec{p}$ boils down to the solution of a linear system as that in \Cref{foot:linear} and in \Cref{lem:identify}. For that, we need $d$ linearly independent vectors $\vec{x}^{(1)}, \ldots, \vec{x}^{(d)} \in S$ and an additional fixed element $\vec{x}^\star \in S$ for the normalization of the probabilities in the right-hand side. With high probability, all $\vec{x}^{(1)}, \ldots, \vec{x}^{(d)} \in H_{\vec{w}}$. If $\vec{x}^\star$ is also in $H_{\vec{w}}$, normalizing the system by $\vec{x}^\star$ results in an ill-conditioned system. 
%
In fact, we can show that the condition number of the system is $\Omega(1/\lambda^\star)$. Therefore, solving the linear system efficiently requires sampling a vector $\vec{x}^\star \not\in H_{\vec{w}}$ for normalization. However, the probability that we sample (and thus, can use for normalization) a vector $\vec{x}^\star \not\in H_{\vec{w}}$ is at most $\lambda^\star$. 

We now proceed with the proof of \Cref{lem:anticon}.

\begin{proof}
Next, we formalize the intuition behind the sketch of the proof. We recall that for a fixed unit vector $\vec{w} \in \reals^d$, we let $H_{\vec{w}} = \{ \vec{x} \in S : \vec{w}^T \vec{x} \in ( c-\lambda, c+\lambda ) \}$. By the definition of $\lambda^*$, for any $\lambda > \lambda^*$, there is a unit vector $\vec{w} \in \reals^d$ and a $c \in \reals$ such that $\Prob_{\vec{x} \sim \D_S}[ \vec{x} \not\in H_{\vec{w}}] < \lambda$, or equivalently, $\Prob_{\vec{x} \sim \D_S}[ \vec{x} \in H_{\vec{w}} ] \geq 1 - \lambda$. 

We recall that we assume without loss of generality that $S$ is normalized so that $\vec{0} \in S$ and $\D_S(\vec{0}) > 0$. In fact, $\vec{0}$ plays the role of the fixed element $\vec{x}^\star$, discussed in the sketch, which we use for normalization. Next, we distinguish between two cases based on whether $\vec{0} \in H_{\vec{w}}$ or not. 


Let us first fix $\lambda > \lambda^\star$ that lies in a small neighborhood of $\lambda^\star$ of radius $\eps$, where $\eps$ is sufficiently small. We will show that for any such $\lambda$ (that satisfies that $1/\lambda$ is (almost) super-polynomial in $d$), we get a sample complexity of order $1/\lambda$. Since this property will hold arbitrarily close to $\lambda^\star$, the sample complexity will be super-polynomial in the dimension $d$.

Having chosen $\lambda$ as above, there is a direction $\vec{w}$ and a translation $c \in \reals,$ that define the space $H_{\vec{w}},$ such that  $\Prob_{\vec{x} \sim \D_S}[ \vec{x} \not\in H_{\vec{w}}] < \lambda.$ There are two cases for the translation $c.$

\textsc{Case A:} We may first assume that $c$ is small enough, that is $|c| < \lambda$ and, hence, $0 \in (c-\lambda, c+\lambda)$. Let $X$ be any set of $O(1/\lambda)$ samples from $\D_S$. Then, with constant probability, all $X \subseteq H_{\vec{w}}$. Let $\vec{X}_d = [ \vec{x}^{(1)}, \ldots, \vec{x}^{(d)} ]^T$ be the matrix obtained by any $d$ elements $\vec{x}^{(1)}, \ldots, \vec{x}^{(d)} \in X$ different from $\vec{0}$. By \Cref{lem:identify}, recovering $\vec{p}$ requires the solution of the linear system $\vec{X}_d \vec{z} = \log(\vec{q})$, where $\log(\vec{q}) = (\log(q_j))_{j \in [d]}$ and $q_j = \frac{\D_S(\vec{x}^{(j)})}{\D_S(\vec{0})}$, for each $j \in [d]$. 

We next show that since $c \in (-\lambda, \lambda)$, with constant probability, the matrix $\vec{X}_d$ is ill-conditioned and has condition number 
\footnote{Let $\vec{A}$ be a $d\times d$ square matrix with singular values $s_{1} \geq \cdots \geq s_{d} \geq 0$. We will denote with $s_{\max}(\vec{A}) = s_{1}$ and with $s_{\min}(\vec{A}) = s_{d}$. The condition number of the $\vec{A}$ is $\kappa(\vec{A}) = s_{\max}(\vec{A})/s_{\min}(\vec{A})$. The condition number $\kappa(\vec{A}) \in [1, \infty]$ quantifies the sensitivity of the solution to a linear system $\vec{A}\vec{z} = \vec{b}$ to the small perturbations of $\vec{b}$.}
$\kappa(\vec{X}_d) = \Omega(1/\lambda)$. 

Specifically, since all $\vec{x}^{(1)}, \ldots, \vec{x}^{(d)}$ are different from $\vec{0}$, there is a unit vector $\vec{w}' \in \reals^d$ so that $\| \vec{X}_d \vec{w}' \|_2 \geq 1$. On the other hand, by the hypothesis that with constant probability, $X \subseteq H_{\vec{w}}$, $\| \vec{X}_d \vec{w} \|_2 \leq (|c|+\lambda) \cdot \sqrt{d} \leq 2\lambda \cdot \sqrt{d}$.
Therefore, the condition number of the matrix $\vec{X}_d$ is $\kappa(\vec{X}_d) = \Omega(1/(\lambda \cdot \sqrt{d}))$ for the fixed $\lambda > \lambda^*$ in the neighborhood of $\lambda^\star$. 
This implies that the condition matrix is of order $\Omega(1/\lambda)$. 
Hence, with constant probability, we cannot recover ($\vec{z}$ and) $\vec{p}$ within accuracy $o(1)$, unless we estimate the right-hand side $\vec{q}$ of the linear system $\vec{X}_d \vec{z} = \log(\vec{q})$ with accuracy $o(\lambda)$, which requires $\omega(1/\lambda)$ samples. 

\textsc{Case B:} Otherwise, if $|c| > \lambda$, then $0 \not\in (c-\lambda, c+\lambda)$. Since $\vec{w}^T \vec{0} = 0$, the probability that $\vec{0}$ is sampled from $\D_S$ is at most $\lambda$. Hence, unless we take $\omega(1/\lambda)$ samples, we cannot find a good estimation of $\D_S(\vec{0})$, which is required for the linear system $\vec{X}_d \vec{z} = \log(\vec{q})$, whose solution recovers ($\vec{z}$ and) $\vec{p}$.

Finally, since either Case A or B will hold for any $\lambda>\lambda^*$ in the $\epsilon$-neighborhood of $\lambda^\star$, we let $\lambda \downarrow \lambda^*$ and hence we get that an expected number of $\Omega(1/\lambda^*)$ samples is required, which is super-polynomial in $d$.

\end{proof}

The above condition highlights a gap between the continuous problem of learning truncated Gaussian distributions \cite{DGTZ18} and the discrete case, where truncation can be quite restrictive.

For the efficient estimation of $\vec{z}$, we also need to assume that the truncation set $S$ is large enough. Namely, we assume that: 

\begin{assumption}\label{as:mass}
For the truncation set $S$, there is a constant $\alpha > 0$ so that the Boolean product  distribution $\D$ has $\D(S) \geq \alpha$. 
\end{assumption}
\Cref{as:mass} is not necessary for efficient learning, in the sense that e.g., there may be
$\alpha$-fat product distributions which do not satisfy this condition, but are still efficiently learnable using \Cref{cor:fat-learning}.

We conclude this section with a remark. Note that complex models, such as Bayes networks and Ising
models, can be cast as truncated product distributions
in a Boolean hypercube of appropriately
high dimension (the translation is conceptually similar to that for Mallows models in \Cref{s:ranking}). For instance, the Ising model over $\{-1,+1\}^d$ with interaction matrix $J$ (with $J_{ii} = 0)$ and external field $h$
is defined by the function
$\pi(x) = x^T J x + h^T x$ and is a probability measure $\mu(x) \propto \exp(\pi(x))$. We have a dimension for each edge and a dimension for each spin (so the Boolean Product distribution is a measure over $\{-1,+1\}^{\binom{d}{2} + d}$) and the truncation set $S_{\mathrm{Ising}}$ consists
of all $2^d$ binary vectors with valid edge labels (i.e., vectors with edge labels consistent with some allocation of $\{+,-\}$ to the vertices). 
Hence, we can consider the product probability measure over the points $x \in \{-1,+1\}^{\binom{d}{2} + d}$ with density
\[
\D(x) = \D((x_{uv})_{u,v \in E}, (x_u)_{u \in V}) = \prod_{(u,v) \in E} \frac{\exp(J_{uv}x_{uv})}{2 \cosh(J_{uv})} \prod_{u \in V} \frac{\exp(h_u x_u)}{2 \cosh(h_u)}\,.
\]
Casting an Ising model $\mu$ to our setting results in a truncated
Boolean product distribution $\mu(x) = \D(x) \vec 1\{x \in S_{\mathrm{Ising}}\}/\D(S_{\mathrm{Ising}})$ that satisfies \Cref{as:identify}, \Cref{as:oracle} and pontentially \Cref{as:anticon}, assuming that the parameters of the
Ising model are ``sufficiently nice'' so that the $\binom{d}{2}+d$ parameters of $\D$ are bounded away from $0$ and $1$. In general, it is not guaranteed to satisfy \Cref{as:mass}, which is in accordance with the fact that sampling from an Ising model is computationally hard in
general (see e.g., \cite{huber1999efficient,sly2012computational})).

In the following section, we present the Projected Stochastic Gradient Descent algorithm and show that assumptions~\ref{as:oracle},~\ref{as:anticon} {(under some particular regime for the parameter $\lambda$)}~and~\ref{as:mass} (e.g., for an absolute constant $\alpha)$ are sufficient for the estimation of the natural parameter vector $\vec{z}$ of the Boolean product distribution $\D$ by sampling from its truncation $\D_S$. 

\section{PSGD for Learning Truncated Boolean Product Distributions}
\label{s:sgd}

We next show how to estimate the natural parameter vector $\vec{z}^\star$ of a Boolean product distribution $\D(\vec{z}^\star)$ using samples from its truncation $\D_{S}(\vec{z}^\star)$, assuming that the true distribution satisfies the conditions~\ref{as:oracle},~\ref{as:anticon}~and~\ref{as:mass}. 

\begin{algorithm}[ht] 
\caption{Projected Stochastic Gradient Descent with Samples from $\D_{S}(\vec{p}^{\star})$.}\label{algo:sgd}
\begin{algorithmic}[1]
\Procedure{\textsc{SGD}}{$M, \eta$}\Comment{$M:$ number of steps, $\eta:$ parameter}
\State $\vec{z}^{(0)} \gets \vec{\hat{z}}$ \Comment{$\vec {\hat{ z}}$ is the empirical estimate of \Cref{lem:ps-p}.}
\For{$t = 1..M$}
   \State Sample $\vec{x}^{(t)}$ from $\D_{S}$
   \Repeat 
      \State Sample $\vec{y}$ from $\D(\vec{z}^{(t-1)})$
   \Until{$\vec{y} \in S$} \Comment{\emph{We assume oracle access to} $S.$}
   \State $\vec{v}^{(t)} \gets -\vec{x}^{(t)} + \vec{y}$
   \State $\vec{z}^{(t)} \gets \Pi_{\mathcal{B}}(\vec{z}^{(t-1)} -  \frac{1}{t \cdot \eta } \vec{v}^{(t)})$ \Comment{$\eta_{t} = 1/(t \cdot \eta)$: step size}
\EndFor
\State \textbf{return} $\vec{\overline{z}}  \gets \frac{1}{M}\sum_{t=1}^{M}\vec{z}^{(t)}$  
\EndProcedure
\end{algorithmic}
\end{algorithm}

Similarly to \cite{DGTZ18}, we use Projected Stochastic Gradient Descent (SGD) on the negative log-likelihood of the truncated samples. Our SGD algorithm is described in \Cref{algo:sgd}. We should highlight that \Cref{algo:sgd} runs in the space of the natural parameters $\vec{z}$ of the Boolean product distribution. Changing the parameters from $\vec{p}$ to $\vec{z}$ results in a linear system, similar to that in \Cref{foot:linear} and in the proof of \Cref{lem:identify} and simplifies the analysis of the log-likelihood function. Furthermore, by \Cref{prop:kl-dtv-norm}, estimating $\vec{z}^\star$ within error at most $\eps$ in $L_{2}$ norm results in a distribution within total variation distance at most $\eps$ to $\D(\vec{z}^\star)$.

Throughout the analysis of \Cref{algo:sgd}, we make use of Assumptions~\ref{as:oracle} - \ref{as:mass}. The technical details of the analysis are deferred to \Cref{app:sgd-an}. The analysis goes as follows: we first derive the negative log-likelihood function that \Cref{algo:sgd} optimizes. Since the truncation set $S$ is only accessed through membership queries, we do not have a closed form of the log-likelihood. 

However, we can show that it is convex for any truncation set $S$. We prove that the natural parameter vector $\vec{\hat{z}}$ corresponding to the empirical estimate $\vec{\hat{p}}_{S}$ is a good initialization for \Cref{algo:sgd}. Specifically, we show that $\vec{\hat{p}}_{S}$ is close to the true parameter vector $\vec{p}^\star$ in $L_{2}$ distance, and that this proximity holds for the corresponding natural parameter vectors as well. 

For the correctness of \Cref{algo:sgd}, it is essential that it runs in a convex region. We can show that there exists a ball $\mathcal{B}$, centered at the initialization point $\vec{\hat{z}}$, which contains $\vec{z}^{\star}$. The radius of the ball depends on the lower bound $\alpha$ of $\D(S)$ (\Cref{as:mass}) {and the parameter $\lambda$ of the anti-concentration condition}. We can prove that Assumptions~\ref{as:anticon}~and~\ref{as:mass} always hold inside $\mathcal{B}$. That is, for any vector $\vec{z} \in \mathcal{B}$ (and the corresponding parameter vector $\vec{p}$), both the anti-concentration assumption and the mass assigned to the truncation set $S$ by $\D_{S}(\vec{p})$ can be lower bounded by a polynomial function of $\alpha$ {and an exponential function of $1/\lambda$ (this is where the exponential dependence on the anti-concentration parameter comes up)}. 

Under these two assumptions, we can prove that the negative log-likelihood is strongly-convex inside the ball $\mathcal{B}$. Hence, while \Cref{algo:sgd} iterates inside $\mathcal{B}$, the truncation set has always constant mass and the negative log-likelihood remains strongly-convex. Consequently, \Cref{algo:sgd} converges to the true vector of natural parameters $\vec{z}^{\star}$. The following theorem  is the main result of the steps described above. 

\begin{theorem} \label{thm:sgd} 
Given oracle access to a measurable set $S \subseteq \Pi_{d}$ (\Cref{as:oracle}), whose measure under some unknown Boolean product distribution $\D(\vec{z}^{\star})$ is at least some constant $\alpha > 0$ (\Cref{as:mass}) and where the truncated distribution $\D_{S}(\vec{z}^{\star})$ satisfies \Cref{as:anticon} with parameter $\lambda$, and given samples from the truncation $\D_S(\vec{z}^{\star})$, there exists a {sample} polynomial-time algorithm that recovers an estimation $\vec{\overline{z}}$ of $\vec{z}^{\star}$. For any $\eps > 0$, the algorithm uses {$\poly(1/\alpha)^{\poly(1/\lambda)} \cdot \widetilde{O}(d/\eps^{2})$ truncated samples from }$\D_S(\vec{z}^{\star})$ and membership queries to $S$ and guarantees that $\|\vec{z}^{\star} - \vec{\overline{z}} \|_{2} \leq \eps$, with probability 99\%. Under these conditions, it also holds that $D_{TV}(\D(\vec{z}^{\star}), \D(\vec{\overline{z}})) \leq O(\eps)$.
\end{theorem}

{Hence, our algorithm is sub-exponential in the dimension $d$ as long as $\exp(1/\lambda)$ is smaller than $\exp(d)$. Showing that there is an algorithm that works for $\lambda = \poly(1/d)$ is an interesting question for future work.}
\medskip
\subsection{Projected SGD: Algorithm's Description}\label{app:sgd-an}
\medskip
\label{subs:sgd-algo}

In this section, we present and explain the Projected SGD algorithm that learns the true natural parameter vector $\vec{z}^{\star}$ and, consequently, as we showed in \Cref{prop:kl-dtv-norm}, learns the true Boolean product distribution $\D(\vec{p}^{\star})$ in total variation distance. 

We are now ready to present the main steps of our SGD \Cref{algo:sgd}. The input of the algorithm is the number of the steps $M$ and a parameter $\eta$, that modifies the step size. The initialization point $\vec{z}^{(0)}$ of the algorithm will be the point $\vec{\hat{z}}$, that equals to the natural parameter vector of the empirical estimate $\vec{\hat{p}}_{S}$, defined by \Cref{eq:emp}. For $t \in [M]$, our guess for the true natural parameter vector $\vec{z}^{\star}$ will be denoted by $\vec{z}^{(t)}$. In each round $t$, we produce a guess $\vec{z}^{(t)}$ as follows: Firstly, we draw a sample $\vec{x}^{(t)}$ from the unknown truncated Boolean product distribution $\D_{S}(\vec{p}^{\star})$. Also, we draw a second sample $\vec{y}$ from the distribution induced by our previous guess $\vec{z}^{(t-1)}$. Note that it is possible that the generated sample $\vec{y}$ does not lie in the truncation set $S$. Hence, we have to iterate until we draw a sample that lies in $S$, that is $M_{S}(\vec{y}) = \mathbf{1}_{\vec{y} \in S}$ is equal to 1. As we have already mentioned, the function that we are minimizing is the negative log-likelihood for the population model. As we will see in \Cref{lem:conv} and \Cref{eq:like}, the true gradient of this function is equal to
\begin{equation*}
    -\E_{\vec{x} \sim \mathcal{D}_S(\vec{z}^{\star})}[\vec{x}] + \E_{\vec{y} \sim \mathcal{D}_S(\vec{z})}[\vec{y}] \,.
\end{equation*}

In \Cref{algo:sgd}, this quantity corresponds to a random direction denoted by $\vec{v}^{(t)}$ at step $t$ and is equal to $-\vec{x}^{(t)} + \vec{y}$. Note that its expected value is equal to the true gradient. Hence, as in the classical gradient descent setting, we update our guess using the following update rule

\begin{equation*}
    \vec{z}^{(t)} \gets \vec{z}^{(t-1)} - \eta_t \vec{v}^{(t)} \,.
\end{equation*}

As we have explained, we perform the SGD algorithm in a ball $\mathcal{B}$ of radius. Hence, it may be the case that our new guess $\vec{z}^{(t)}$ lies outside $\mathcal{B}$. Hence, we have to project that point back to the ball. For that reason, we use the projection function $\Pi_{\mathcal{B}}$, that equals to the mapping
\begin{equation*}
    \Pi_{\mathcal{B}}(\vec{x}) = \argmin_{\vec{z} \in \mathcal{B}} \| \vec{x} - \vec{z}\|_{2} \text{ for } \vec{x} \in \mathbb{R}^{d}\,.
\end{equation*}

Finally, after $M$ steps, the SGD algorithm returns an estimate $\overline{\vec{z}}$ that is close to the minimizer of the negative log-likelihood function. As we will show, this minimizer corresponds to the true natural parameters vector $\vec{z}^{\star}$. In the next section, we perform the theoretical analysis of the projected stochastic gradient descent algorithm for truncated Boolean product distributions.



\section{Projected SGD: Theoretical Analysis}
Our goal is to prove \Cref{thm:sgd}. The roadmap of the proof is presented as follows:
\begin{itemize}
    \item \textbf{Convexity of the objective.} In \Cref{subs:log-like}, we show that the population version of the negative log-likelihood objective is convex with respect to the natural parameter vector (see \Cref{lem:conv} and \Cref{subsec:nll-pop}).
    \item \textbf{Initial feasible point.} In \Cref{subs:pslemma}, we efficiently compute a good initialization point for the SGD algorithm.
    The statement is presented in \Cref{lem:ps-p}.
    \item \textbf{Feasible region.} In \Cref{subs:ball}, we show that there exists a ball (and hence an easy-to-project set) that contains the true vector $\vec{z}^{\star}$ (see \Cref{lem:ball}) and each point in the ball satisfies Assumptions~\ref{as:anticon} (see \Cref{lem:aconc-ball}) and~\ref{as:mass} (see \Cref{lem:mass-ball}).
    \item \textbf{Unbiased estimation of the gradient.} In \Cref{subs:grad-estim}, we show how to obtain an unbiased estimation of the gradient of the objective efficiently.
    \item \textbf{Strong convexity inside the feasible region.} In \Cref{subs:strconv}, we establish that the negative log-likelihood objective is strongly-convex inside the ball of \Cref{subs:ball}.
    \item \textbf{Analysis of the SGD algorithm.} In \Cref{subs:sgd-an}, we show that the bounded variance step property holds (see \Cref{lem:bound-var}). Hence, combining this result with the strong-convexity inside the ball, we can apply \Cref{thm:sgd-main} and get \Cref{thm:sgd}.
\end{itemize}

\subsection{Convexity of the negative log-likelihood}
\label{subs:log-like}

Let $S$ be a subset of the hypercube $\Pi_{d}$ and $\D(\vec{p})$ be an arbitrary Boolean product distribution. We remind the reader that, for $\vec{x} \in \Pi_{d}$:
\begin{equation*}
    \D(\vec{p};\vec{x}) = \Be(p_{1};x_{1}) \otimes \dots \otimes \Be(p_{d};x_{d}) = \prod_{i \in [d]}(p_{i}^{x_{i}}(1-p_{i})^{1-x_{i}})\,. 
\end{equation*}

Let $\vec{z}$ be the natural parameters vector with $z_{i} = \ln \frac{p_{i}}{1-p_{i}}$ for $i \in [d]$.
Rewriting the distribution as an exponential family, we get that:
\begin{equation*}
    \D(\vec{p} ; \vec{x}) = \prod_{i \in [d]}\exp \Big( x_{i} \ln\frac{p_{i}}{1-p_{i}} + \ln(1-p_{i}) \Big)\,,
\end{equation*}
or equivalently:
\begin{equation*}
    \D(\vec{z} ; \vec{x}) =  \frac{\exp(\vec{x}^{T}\vec{z})}{\prod_{i \in [d]}(1+\exp(z_{i}))} \,.
\end{equation*}
The truncation set $S$ induces a distribution $\D_{S}(\vec{z})$, that is equal to:
\begin{equation*}
    \mathcal{D}_{S}(\vec{z}; \vec{x}) = \vec{1}_{\vec{x} \in S}\frac{\exp(\vec{x}^{T}\vec{z})}{\sum_{\vec{y} \in S} \exp(\vec{y}^{T}\vec{z}) }\,.
\end{equation*}

Afterwards, we compute the negative log-likelihood $\ell(\vec{z})$ of the truncated samples drawn from the truncated distribution $\D_{S}(\vec{z})$ and study its behavior in terms of convexity. 

\subsubsection{Log-likelihood for a Single Sample}
Notice that the structure of the truncated Boolean product distribution $\D_{S}(\vec{z})$, expressed as an exponential family, is quite useful when computing the negative log-likelihood for a single sample $\vec{x}$ drawn from a distribution $\mathcal{D}_{S}(\vec{z})$, that is:
\begin{equation} \label{eq:log}
    \ell(\vec{z};\vec{x}) = -\ln \D_{S}(\vec{z};\vec{x}) = -\vec{x}^{T}\vec{z} + \ln \Big (\sum_{\vec{y} \in S} e^{\vec{y}^{T}\vec{z}} \Big )\,.
\end{equation}

The convexity of the negative log-likelihood $\ell(\vec{z})$ of the truncated Boolean product distribution $\D_{S}(\vec{z})$ follows immediately if one computes the gradient and the Hessian of $\ell(\vec{z})$ with respect to the natural parameter vector $\vec{z}$. This result is presented in the following Lemma.
\begin{lemma} \label{lem:conv}
The negative log-likelihood objective $\ell(\vec{z};\vec{x})$, as defined in \Cref{eq:log}, is convex with respect to $\vec{z}$ for all $\vec{x} \in \Pi_{d}$.
\end{lemma}

\begin{proof} 
Observe that the negative log-likelihood of a single sample $x \sim \mathcal{D}_{S}(\vec{z})$ will be
\begin{equation*}
    \ell(\vec{z} ; \vec{x}) = -\vec{x}^{T}\vec{z} + \ln \Big(\sum_{\vec{y} \in S} e^{\vec{y}^{T}\vec{z}} \Big)\,.
\end{equation*}
We now compute the gradient of $\ell(\vec{z};\vec{x})$ with respect to the parameter $\vec{z}$. 
\begin{equation*}
    \nabla_{\vec{z}} \ell(\vec{z} ; \vec{x}) = -\vec{x} + \frac{\sum_{\vec{y} \in S} \vec{y}e^{\vec{y}^{T}\vec{z}}}{\sum_{\vec{y} \in S} e^{\vec{y}^{T}\vec{z}}} = -\vec{x} + \E_{\vec{y} \sim \mathcal{D}_{S}(\vec{z})}[\vec{y}] \,.
\end{equation*}
Finally, we compute the Hessian of the negative log-likelihood:
\begin{equation*}
    \vec{H}_{\ell}(\vec{z}) = \frac{\sum_{\vec{y} \in S} \vec{y}\vec{y}^{T}e^{\vec{y}^{T}\vec{z}}}{\sum_{\vec{y} \in S} e^{\vec{y}^{T}\vec{z}}} - \frac{\sum_{\vec{y} \in S} \vec{y}e^{\vec{y}^{T}\vec{z}}}{\sum_{\vec{y} \in S} e^{\vec{y}^{T}\vec{z}}} \frac{\sum_{\vec{y} \in S} \vec{y}e^{\vec{y}^{T}\vec{z}}}{\sum_{\vec{y} \in S} e^{\vec{y}^{T}\vec{z}}} = \text{Cov}_{\vec{y} \sim \mathcal{D}_{S}(\vec{z})}[\vec{y}, \vec{y}]\,.
\end{equation*}
The Hessian of the negative log-likelihood $\vec{H}_{\ell}$ is semi-positive definite since it equals to a covariance matrix (in particular, it equals to the covariance matrix of the sufficient statistics of the exponential family). The result follows.
\end{proof}

\subsubsection{Log-likelihood for the Population Model}
\label{subsec:nll-pop}
Our Projected SGD algorithm will optimize the negative log-likelihood for the population model, that will be denoted with $\overline{\ell}$. This function is defined as the expected value of the negative log-likelihood function with respect to the true truncated Boolean product distribution $\D_{S}(\truez)$, that is
\begin{equation*}
    \overline{\ell}(\vec{z}) = \E_{\vec{x} \sim \D_{S}(\truez)}[\ell(\vec{z};\vec{x})]\,.
\end{equation*}
Using the formula of \Cref{eq:log}, we get that
\begin{equation*}
    \overline{\ell}(\vec{z}) = \E_{\vec{x} \sim \mathcal{D}_{S}(\vec{z}^{\star})} \Big [-\vec{x}^{T}\vec{z} + \ln \Big(\sum_{\vec{y} \in S} e^{\vec{y}^{T}\vec{z}} \Big) \Big]\,.
\end{equation*}
But, since the second term is just a normalization constant, and hence independent of the random variable $\vec{x}$, we get that:
\begin{equation*}
    \overline{\ell}(\vec{z}) = \E_{\vec{x} \sim \mathcal{D}_{S}(\vec{z}^{\star})}[-\vec{x}^{T}\vec{z}] + \ln \Big(\sum_{\vec{y} \in S} e^{\vec{y}^{T}\vec{z}} \Big)\,.
\end{equation*}
Similarly, as in the proof of \Cref{lem:conv}, one can compute the gradient with respect to $\vec{z}$ and get that:
\begin{equation} \label{eq:like}
    \nabla_{\vec{z}}  \overline{\ell}(\vec{z}) =  -\E_{\vec{x} \sim \mathcal{D}_{S}(\vec{z}^{\star})}[\vec{x}] + \E_{\vec{y} \sim \mathcal{D}_{S}(\vec{z})}[\vec{y}]\,. 
\end{equation}
Hence, computing in the exact same way the Hessian of $\overline{\ell}(\vec{z})$, we get the convexity of the
negative log-likelihood for the population model
with respect to the natural parameter vector $\vec{z}$. 

Also, notice that the gradient $\nabla_{\vec{z}}  \overline{\ell}(\vec{z})$ vanishes when $\vec{z} = \vec{z}^{\star}$. So, the true parameter vector $\vec{z}^{\star}$ minimizes the negative log-likelihood function of the truncated samples for the population model. This fact combined with the convexity of the population version of the negative log-likelihood yield the following.
\medskip
\begin{lemma} 
For any $\vec{z} \in \mathbb{R}^{d}$, it holds that
\begin{equation*}
    \overline{\ell}(\vec{z}^{\star}) \leq \overline{\ell}(\vec{z})\,,
\end{equation*}
where $\vec z^\star \in \reals^d$ is the true parameter vector and $\overline{\l}$ is the population negative log-likelihood objective, whose expectation is with respect to the truncated Boolean product distribution $\D_S(\vec z^\star)$, for some arbitrary truncation set $S \subseteq\Pi_d$.
\end{lemma}

\subsection{\upshape Initialization Lemma}
\label{subs:pslemma}
Our next goal is to find a good initialization point for our SGD algorithm. Assume that for the truncation set $S$, it holds that $\D(\truep ; S) = \alpha$.  We claim that, if one draws $n = \widetilde{O}(d)$ samples $\{\vec{x}^{(t)}\}_{t=1}^{n}$ from the truncated Boolean product distribution $\D_{S}(\truep)$, the empirical mean 
\begin{equation} \label{eq:emp}
    \vec{\hat{p}}_{S} = \frac{1}{n}\sum_{t=1}^{n}\vec{x}^{(t)}
\end{equation}
is close in $L_{2}$ distance to the true mean parameter vector $\truep$ with high probability. 

In the following lemma, we provide the proximity result between the empirical mean $\hat{\vec{p}}_{S}$ of the truncated Boolean product distribution $\D_{S}(\truep)$ and the true parameter vector $\vec{p}^{\star}$. This lemma will be useful in the upcoming section. 
\begin{lemma} \label{lem:ps-p}
Let $\D(\truep)$ be the unknown Boolean product distribution and consider the truncation set $S \subseteq\Pi_{d}$ such that $\D(\truep;S) = \alpha$. The empirical mean $\hat{\vec{p}}_{S}$, computed using $O\left (d \ln(\frac{d}{\delta}) \right)$  samples from the truncated Boolean product distribution $\D_{S}(\truep)$, satisfies:
\begin{equation*}
    \| \hat{\vec{p}}_{S} - \vec{p}^{\star}\|_{2} \leq O \Big( \sqrt{\ln(1/\alpha)} \Big)\,,
\end{equation*}
with probability $1-\delta$.
\end{lemma} 
\begin{proof}
The proof of \Cref{lem:ps-p} can be decomposed in the following two lemmas. Combining the following two lemmas (we apply \Cref{lem:ps1} with accuracy $\epsilon$ a small constant like $\sqrt{\log(1/\alpha)}/10$, since $\alpha$ is also a constant) using the triangle inequality for the $L_{2}$ norm, \Cref{lem:ps-p} follows.
\end{proof}

\begin{lemma} \label{lem:ps1}
Consider $S \subseteq\Pi_{d}$ and let $\vec{p}_{S}$ be the parameter vector of the truncated Boolean product distribution $\mathcal{D}_{S}(\vec{p}^{\star})$. There exists an algorithm that uses $O(\frac{d}{\epsilon^{2}}\ln(\frac{d}{\delta}))$ samples from $\D_{S}(\vec{p}^{\star})$ and computes an estimate $\hat{\vec{p}}_{S}$ such that
\begin{equation*}
    \|\hat{\vec{p}}_{S} - \vec{p}_{S}\|_{2} \leq \epsilon\,, 
\end{equation*}
with probability $1-\delta$.
\end{lemma}

\begin{proof}
Consider the truncated true Boolean product distribution $\D_{S}(\vec{p}^{\star})$ with truncation set $S \subseteq\Pi_{d}$. Consider the algorithm that, given $n$ samples $\{\vec{x}^{(t)}\}$ from $\D_{S}(\vec{p}^{\star})$, computes the empirical mean vector:
\begin{equation*}
    \hat{\vec{p}}_{S} = \frac{1}{n}\sum_{t=1}^{n}\vec{x}^{(t)}\,.   
\end{equation*}
Note that $\E\hat{\vec{p}}_{S} = \vec{p}_{S}$. Fix a coordinate $j \in [d]$.
By applying Hoeffding's  inequality at $\hat{p}_{S,j} = \frac{1}{n}\sum_{t=1}^{n}x^{(t)}_{j}$ (these random variables are bounded in $[0,1]$), one gets
\begin{equation*}
    \Pr \Big [|\hat{p}_{S,j} - p_{S,j}| > \epsilon / \sqrt{d} \Big ] \leq 2e^{-2 n \frac{\epsilon^{2}}{d}}\,.
\end{equation*}
We can now use union bound and require the left hand side to be at most $\delta$. Hence, we get that
\begin{equation*}
    2de^{-2 n \frac{\epsilon^{2}}{d}} \leq \delta \Rightarrow n = \Omega \left (\frac{d}{\epsilon^{2}}\ln \left(\frac{d}{\delta} \right) \right)\,.
\end{equation*}
Consequently, given $\Theta(\frac{d}{\epsilon^{2}}\ln(\frac{d}{\delta}))$ samples, we get that the empirical mean estimate $\vec{p}_{S}$ is within error $\epsilon$ in $L_{2}$ distance with probability $1-\delta$. 
 \end{proof}

\begin{lemma} \label{lem:ps2}
Consider the unknown Boolean product distribution $\D(\vec{p}^{\star})$ and a truncation set $S$ such that $\D(\truep; S) = \alpha$. Let $\vec{p}_{S}$ be the parameter vector of the truncated Boolean product distribution $\mathcal{D}_{S}(\vec{p}^{\star})$. Then, it holds that
\begin{equation*}
  \| \vec{p}_{S} - \vec{p}^{\star} \|_{2} \leq O\Big(\sqrt{\ln(1/\alpha)}\Big)\,.
\end{equation*}
\end{lemma}

\begin{proof}
Consider an arbitrary direction $\vec{w}$ with $\|\vec{w}\|_{2} = 1$. Consider the random variable $\vec{w}^{T}\vec{x}$ where $\vec{x} \sim \mathcal{D}(\vec{p}^{\star})$. Note that $\E_{\vec{x} \sim \mathcal{D}(\vec{p}^{\star})}[\vec{w}^{T}\vec{x}] = \vec{w}^{T}\vec{p}^{\star}$. By applying Hoeffding's inequality:
\begin{equation*}
    \Pr_{\vec{x} \sim \mathcal{D}(\vec{p}^{\star})}[\vec{w}^{T}  \vec{x} > \vec{w}^{T} \vec{p}^{\star} + C] \leq e^{-2C^{2}}\,.
\end{equation*}
Hoeffding's inequality implies that the marginal of the true distribution in direction $\vec{w}$ has exponential tail and that holds for any (unit) direction. But, the worst case set $S$ would assign mass $\alpha$ to the tail (in order to maximize the distance between the two means) and, hence:
\begin{equation*}
    \alpha \leq  e^{-2C^{2}} \Rightarrow C = O\left(\sqrt{\ln \frac{1}{\alpha}}\right)\,.
\end{equation*}
The result follows.

 \end{proof}

\subsection{\upshape Ball in the $z$-space}
\label{subs:ball}

We will perform Projected SGD in a convex subspace of $\mathbb{R}^{d}$. The algorithm will optimize the negative log-likelihood for the population model $ \overline{\ell}$ with respect to the natural parameters $\vec{z} = (z_{1}, \ldots, z_{d})^{T}$ with $z_{i} = \ln \frac{p_{i}}{1-p_{i}}$ in order to learn the true parameters $\truez = (z_{1}^{\star}, \ldots, z_{d}^{\star})^{T}$ with $z_{i}^{\star} = \ln\frac{p_{i}^{\star}}{1-p_{i}^{\star}}$. Our initial guess is $\vec{\hat{z}} = (\hat{z}_{1}, \ldots, \hat{z}_{d})^{T}$ with $\hat{z}_{i} = \ln \frac{\hat{p}_{S,i}}{1-\hat{p}_{S,i}}$. Afterwards, SGD will iterate over estimations $\vec{z}$ of the true parameters $\vec{z}^{\star}$. 

In this section, we show that there exists a convex set that contains the true vector $\vec{z}^{\star}$ and each point in that set satisfies Assumptions~\ref{as:anticon} and~\ref{as:mass} {for some function of $\alpha, \lambda$}. 

In fact, we show that there exists a ball $\ball$ of radius $\rad$ centered at $\vec{\hat{z}}$, that contains the true natural parameters $\vec{z}^{\star}$, with high probability. Additionally, every point $\vec{z}$ of that ball satisfies Assumptions~\ref{as:anticon} and~\ref{as:mass}. That is, for any $\vec{z} \in \ball$, let $\D(\vec{z})$ be the Boolean product distribution and $\D_{S}(\vec{z})$ be an arbitrary truncation of $\D(\vec{z})$. Then, $\D_{S}(\vec{z})$ will be anti-concentrated with some parameter that depends on $\alpha$ and $\lambda$, in the sense of \Cref{as:anticon}, and {we will have $\D(\vec{z};S) > c_{\alpha,\lambda}$ for some constant $c_{\alpha,\lambda}$, that depends again only on the initial mass of the set $S$ and the anti-concentration parameter $\lambda$}. The existence of such a ball is presented in the following lemma.

\medskip
\begin{lemma} 
\label{lem:ball}
There exists $\rad > 0$ such that the ball centered at the empirical estimate $\vec{\hat{z}}:$
\begin{equation*}
    \ball = \{\vec{z} : \| \vec{z} - \vec{\hat{z}} \|_{2} \leq \rad \}
\end{equation*}
contains the true natural parameters, i.e.,
\begin{equation*}
    \| \vec{z}^{\star} - \vec{\hat{z}} \|_{2} \leq \rad\,,
\end{equation*}
with high probability, where the randomness is over the estimate $\vec{\hat{z}}$. {Moreover, the radius $B = B_{\alpha,\lambda} = \poly(1/\lambda) \sqrt{\log(1/\alpha)}.$}
\end{lemma}
\begin{proof}
We can assume that the real mean vector $\vec{p}^{\star}$ lies in $(0,1)^{d}$. Firstly, note that $\vec{\hat{z}} \in (-\infty, \infty)^{d}$, since $(\vec{\hat{z}})_{i} = \ln \frac{\hat{p}_{S,i}}{1-\hat{p}_{S,i}} $ and $0 < \hat{p}_{S,i} < 1$ for any $i \in [d]$. From now on, fix a coordinate $i \in [d]$ and consider the mapping $f(x) = \ln \frac{x}{1-x}$ for $x \in (0,1)$. Note that $f$ corresponds to the transformation of $p_{i}$ to the natural parameter $z_{i}$ and, hence:
\begin{equation*}
    |z^{\star}_{i} - \hat{z}_{i}| = |f(p^{\star}_{i}) - f(\hat{p}_{S,i})|\,.
\end{equation*}
{Using the anti-concentration condition (see \Cref{as:anticon} with parameter $\lambda$), we get that there exists a positive constant $\gamma = \gamma(\lambda)$ such that $p_{i}^{\star}, \hat{p}_{S,i} \in (\gamma, 1-\gamma)$ for any $i \in [d]$ (in particular, $\gamma \geq \lambda^3$ using \Cref{as:anticon} with Chebyshev's inequality). Then, observe that there exists a positive finite constant $C =  C(\lambda)$ such $f$ is $C$-Lipschitz in  that interval (in particular, on the interval $(\gamma,1-\gamma)$, $f$ is $1/(\gamma(1-\gamma))$-Lipschitz)}. Hence, 
\begin{equation*}
    |z^{\star}_{i} - \hat{z}_{i}| = |f(p^{\star}_{i}) - f(\hat{p}_{S,i})| \leq C(\lambda) \cdot |p^{\star}_{i} - \hat{p}_{S,i}|\,.
\end{equation*}
Squaring each side and summing over $i \in [d]$, we get that
\begin{equation*}
    {\|\vec{z}^{\star} - \vec{\hat{z}}\|_{2} \leq \poly(1/\lambda) \cdot \sqrt{\ln \frac{1}{\alpha}}}\,,
\end{equation*}
with high probability, where we used the proximity \Cref{lem:ps-p}.
{Let us set
\[
B_{\alpha,\lambda} = \poly(1/\lambda) \cdot \sqrt{\ln \frac{1}{\alpha}}\,.
\]}
Hence, the ball centered at {$B_{\alpha,\lambda}$}, i.e., the set 
\begin{equation*}
    \ball = \{\vec{z} : \|\vec{z} - \vec{\hat{z}} \|_{2} \leq  {B_{\alpha,\lambda}} \}
\end{equation*}
contains the true natural parameters $\truez$ and any point $\vec{z} \in \ball$ is finite in each coordinate, {since $\sum_{i=1}^{d}(z_{i} - \hat{z}_{i})^{2} \leq B_{\alpha,\lambda}^{2}$.}
 \end{proof}

From now on, we will denote by $\ball$ the ball of \Cref{lem:ball}. In order to be able to perform the SGD algorithm, we have to prove that Assumptions~\ref{as:anticon} and~\ref{as:mass} hold for any guess of our algorithm. Since the algorithm runs inside the ball $\ball$, we have to prove that the two assumptions are preserved inside the ball. We remind the reader that any guess that lies outside the ball, is efficiently projected to its $L_{2}$ closest point $\vec{y} \in \ball$.

Next, in \Cref{lem:mass-ball}, we prove that, in each iteration, every natural parameter vector $\vec{z}$ inside the ball $\ball$, that corresponds to a mean vector $\vec{p}$ and induces a distribution $\D(\vec{p})$, will assign constant non-trivial mass to the set $S$.

\medskip
\begin{lemma} 
[Non-trivial mass inside the ball]
\label{lem:mass-ball}
Consider the true Boolean product distribution $\mathcal{D}(\vec{p}^{\star})$ and $\mathcal{D}(\vec{p})$ be another Boolean product distribution such that the corresponding natural parameter vectors satisfy
\begin{equation*}
    \| \vec{z}^{\star} - \vec{z} \|_{2} \leq B_{\alpha,\lambda}\,.
\end{equation*}
Suppose that for a truncation set $S$ we have that:
\begin{equation*}
    \E_{\vec{x} \sim \D(\vec{p}^{\star})}[\vec{1}_{\vec{x} \in S}] \geq \alpha\,.
\end{equation*}
Then, it holds that
\begin{equation*}
    \E_{\vec{x} \sim \D(\vec{p})}[\vec{1}_{\vec{x} \in S}] \geq {c_{\alpha,\lambda} := \alpha^{ \poly(1/\lambda)}}\,.    
\end{equation*}
\end{lemma}
\begin{proof} Let $\D(\vec{p}^{\star};S) = \alpha$ and $\D(\vec{p};S) = \alpha'$. Firstly, notice that one can express the mass of the set $S$ assigned by $\D(\vec{p})$ as:
\begin{equation*}
    \D(\vec{p};S) = \E_{\vec{x} \sim \D(\vec{p}^{\star})}\left [\vec 1_{\vec{x} \in S} \frac{\D(\vec{p};\vec{x})}{\D(\vec{p}^{\star};\vec{x})} \right]\,.
\end{equation*}
This is equivalent to:
\begin{equation*}
    \D(\vec{p};S) = \E_{\vec{x} \sim \D(\vec{p}^{\star})} \left[ e^{-\ln\frac{\D(\vec{p}^{\star};\vec{x})}{\D(\vec{p};\vec{x})}}\vec 1_{\vec{x} \in S}\right ]\,.
\end{equation*}
We remind the reader that:
\begin{equation*}
    \D(\vec{z} ; \vec{x}) = \exp(\vec{x}^{T}\vec{z}) \frac{1}{\prod_{i \in [d]}(1+\exp(z_{i}))} \,.
\end{equation*}
Writing the log ratio in terms of the natural parameters $\vec{z}$, we get that:
\begin{equation} \label{eq:ant}
    \ln\frac{\D(\vec{z}^{\star};\vec{x})}{\D(\vec{z};\vec{x})}  = \vec{x}^{T}(\vec{z}^{\star}-\vec{z}) + C \,,
\end{equation}
where $C = -\ln \prod_{i \in [d]}(1+e^{z^{\star}_{i}}) + \ln \prod_{i \in [d]}(1+e^{z_{i}}) = \ln \frac{\prod_{i \in [d]} (1-p^{\star}_{i})}{\prod_{i \in [d]} (1-p_{i})}$ is independent of $\vec{x} \sim \D(\vec{p}^{\star})$. Since both $\vec{z}$ and $\vec{z}^{\star}$ lie inside the ball $\mathcal{B}$ and are finite, $C$ corresponds to a constant. 
Now, set $g(\vec{x}) = \ln\frac{\D(\vec{p}^{\star};\vec{x})}{\D(\vec{p};\vec{x})}$ and observe that:
\begin{equation*}
\E_{\vec{x} \sim \D(\vec{p}^{\star})}[g(\vec{x})] = D_{KL}(\D(\vec{p}^{\star}) \parallel \D(\vec{p}))\,.   
\end{equation*}
Using Hoeffding's inequality on Equation (\ref{eq:ant}), we get that:
\begin{equation*}
    \Pr_{\vec{x} \sim \D(\vec{p}^{\star})}\Big[ g(\vec{x}) - \E g \geq t \Big ] \leq \exp(-2t^{2}/\|\vec{z}^{\star}-\vec{z}\|_{2}^{2})\,.
\end{equation*}
Setting $t = \sqrt{\ln(2/\alpha) \| \vec z^{\star} - \vec z\|_{2}^{2} }$, it follows that:
\begin{equation*}
    \Pr_{\vec{x} \sim \D(\vec{p}^{\star})} \left [ g(\vec{x}) - \E g \geq  \sqrt{\ln(2/\alpha)\|\vec{z}^{\star}-\vec{z}\|_{2}^{2}} \right ] \leq  \alpha/2\,.
\end{equation*}
So, with probability at least $1-\alpha/2$, we get that the ratio $-g(\vec{x}) = -\ln\frac{\D(\vec{p}^{\star};\vec{x})}{\D(\vec{p};\vec{x})}$ will be at least
\begin{equation*}
    -\E g - \sqrt{\ln(2/\alpha)\|\vec{z} - \vec{z}^{\star}\|_{2}^{2} }\,,
\end{equation*}
where we have that $\E g = D_{KL}(\D(\vec{p}^{\star}) \parallel \D(\vec{p})) \leq B^{2},$ by \Cref{prop:kl-dtv-norm}.$(i)$.

Hence,  with probability at least $1-\alpha/2$, we get that the ratio $-\ln\frac{\D(\vec{p}^{\star};\vec{x})}{\D(\vec{p};\vec{x})}$ will be at least
{$-B^{2}-B\sqrt{\ln(2/\alpha)}$. Hence, $\alpha' \geq c_{\alpha,\lambda} = \frac{\alpha}{2} \exp(-B^2 - B \sqrt{\log(2/\alpha)})$.} This concludes the proof {by setting
\[
c_{\alpha,\lambda} \geq \alpha^{ \poly(1/\lambda)}\,.
\]
}
 \end{proof}

Applying the above lemma for the initial guess $\vec{\hat{p}}_{S}$, we get that:

\begin{corollary}
Consider a truncated Boolean product distribution $\D_{S}(\vec{p}^{\star})$ with mass $\D(\vec{p}^{\star};S) \geq \alpha > 0$. The empirical mean $\vec{\hat{p}}_{S}$, obtained by \Cref{lem:ps-p}, satisfies {$\D(\vec{\hat{p}}_{S}; S) \geq c_{\alpha,\lambda}$, with high probability, for $c_{\alpha,\lambda} \geq \alpha^{ \poly(1/\lambda)}$.} The high probability result is over the randomness of the initialization $\vec {\hat{p}}_S.$
\end{corollary}

Hence, both at the initialization point $\vec{\hat{z}}$ and while moving inside the ball $\ball$ of \Cref{lem:ball}, the mass assigned to the set $S$ is always non-trivial.

We also need to show that the anti-concentration assumption is valid inside the ball $\mathcal{B}$. \Cref{as:anticon} states that the truncated distribution $\D_{S}(\vec{p}^{\star})$ of the true parameters is anti-concentrated. We will show that this holds for every truncated distribution $\D_{S}(\vec{z})$, induced by $\vec{z}$ that lies inside the ball $\ball$. This is proven by the following lemma.

\begin{lemma} 
[Anti-concentration inside the ball]
\label{lem:aconc-ball}
Consider the true Boolean product distribution $\mathcal{D}(\vec{p}^{\star})$ and $\mathcal{D}(\vec{p})$ be another Boolean product distribution such that the corresponding natural parameter vectors satisfy:
\begin{equation*}
    \| \vec{z}^{\star} - \vec{z} \|_{2} \leq {B_{\alpha,\lambda} =  \poly(1/\lambda) \cdot \sqrt{\ln(1/\alpha)}}\,.
\end{equation*}
Consider an arbitrary truncation set $S \subseteq\Pi_{d}$ such that $\D(\vec{p}^{\star};S) \geq \alpha$.
Assume that \Cref{as:anticon} holds for the true truncated distribution $\D_{S}(\vec{p}^{\star})$ with constant $\lambda$. Then, \Cref{as:anticon} still holds for $\D_{S}(\vec{p})$ with parameter {$\poly(\alpha)^{\poly(1/\lambda)}$}.
\end{lemma}

\begin{proof}
Consider the true Boolean product distribution $\D(\vec{p}^{\star})$. Let $S$ be the truncation set, where $\D(\vec{p}^{\star};S) = \alpha$. 
The true truncated Boolean product distribution $\D_{S}(\vec{p}^{\star})$ satisfies \Cref{as:anticon}. Hence, there exists a $\lambda$, such that, for any arbitrary hyperplane defined by $\vec{w} \in \mathbb{R}^{d}$ with $\|\vec{w}\|_{2} = 1$ and $c \in \mathbb{R}$, we have that $\D_{S}(\vec{p}^{\star};H) = \lambda$, where $H = \{\vec{x} : \vec{w}^T \vec{x} \not\in ( c-\lambda, c+\lambda )\} \subseteq\Pi_{d}$. Hence, the mass assigned by the true Boolean product distribution to the space $H \cap S$ is equal to $\D(\vec{p}^{\star}; H \cap S) = \lambda \alpha$.

Now, note that \Cref{lem:mass-ball} holds for arbitrary set $S$. Hence, we can take the set to be equal to $H \cap S$. Then, note that the hypotheses of \Cref{lem:mass-ball} hold with $\D(\vec{p}^{\star}; H \cap S) \geq \lambda \alpha$. {Applying the result of \Cref{lem:mass-ball}, we get that: $\D(\vec{p};H \cap S) = \poly(\alpha)^{\poly(1/\lambda)}$. 
Hence,  $\D_{S}(\vec{p}; H )= \poly(\alpha)^{\poly(1/\lambda)}$.} 
 
\end{proof}
Applying the above lemma for the initial guess $\vec{\hat{p}}_{S}$, we get that:
\begin{corollary}
Consider a truncated Boolean product distribution $\D_{S}(\vec{p}^{\star})$ for which \Cref{as:anticon} holds with parameter $\lambda$. The truncated Boolean product distribution $\D_{S}(\vec{\hat{p}}_{S})$ induced by the empirical mean $\vec{\hat{p}}_{S}$, obtained by \Cref{lem:ps-p}, satisfies \Cref{as:anticon} with parameter $\poly(\alpha)^{\poly(1/\lambda)}$, with high probability over the randomness of the initialization $\vec {\hat{p}}_S.$
\end{corollary}

Hence, any natural parameter vector $\vec{z} \in \ball$, induces a distribution $\D(\vec{z})$ such that the truncated distribution $\D_{S}(\vec{z})$ satisfies the anti-concentration assumption with some parameter that depends only on $\alpha$ and $\lambda$.

\subsection{\upshape Unbiased Estimation of the Gradient}
\label{subs:grad-estim}
In this section, we discuss the rejection sampling algorithm in order to obtain an unbiased estimate for the gradient of the population version of the negative log-likelihood objective. Recall that
\begin{equation*}
    \nabla_{\vec{z}}  \overline{\ell}(\vec{z}) =  -\E_{\vec{x} \sim \mathcal{D}_{S}(\vec{z}^{\star})}[\vec{x}] + \E_{\vec{y} \sim \mathcal{D}_{S}(\vec{z})}[\vec{y}]\,. 
\end{equation*}
To compute an unbiased estimate for the first term, it suffices to draw a single sample from the distribution $\mathcal{D}_{S}(\vec{z}^{\star})$ (we have oracle sample access to this distribution). For the second term, we perform rejection sampling as follows: we draw a vector $\vec{y} \sim \mathcal{D}(\vec{z})$ and we check whether $\vec y \in S$, using the membership oracle access to the truncation set $S.$ If $\vec y$ lies in $S$, we use it to obtain the unbiased gradient estimate; otherwise, we reject this sample and repeat the procedure.
We remind the reader that in each iteration we project the guess vector back to the feasible region $\ball.$ Since the mass of the set $S$ inside the ball $\ball$ is non-trivial, we get that the rejection sampling algorithm takes {$1/c_{\alpha,\lambda} = 1/\alpha^{\poly(1/\lambda)}$} samples from the Boolean product distribution $\mathcal{D}(\vec{z})$ with high probability.

\subsection{\upshape Strong-convexity of the negative log-likelihood}
\label{subs:strconv}
A crucial ingredient of our SGD algorithm is the strong convexity of $ \overline{\ell}(\vec{z}), $ that is the negative log-likelihood for the population model that corresponds to the truncated Boolean product distribution $\D_{S}(\vec{z})$. Specifically:

\medskip
\begin{definition}
Let $f:\mathbb{R}^{d} \rightarrow \mathbb{R}$ with Hessian matrix $\vec{H}_{f}$. Then, $f$ will be called $\lambda$-strongly convex if it holds that $\vec{H}_{f} \succeq \lambda \mathbb{I}$.
\end{definition}
As a last step before the analysis of our SGD algorithm, we will use \Cref{lem:str-conv} to show that $ \overline{\ell}(\vec{z})$ is strongly convex for any $\vec{z} \in \ball$, {where $\ball$ is the projection set of \Cref{lem:ball}}. Let $\vec{H}_{ \overline{\ell}}$ be the corresponding Hessian of $ \overline{\ell}$ with the presence of arbitrary truncation $S \subseteq\Pi_{d}$.

\medskip
\begin{lemma} 
[Strong Convexity]
\label{lem:str-conv}
Consider an arbitrary truncation set $S \subseteq\Pi_{d}$ whose mass with respect to the true Boolean product distribution is  $\D(\vec{p}^{\star};S) = \alpha$ and the truncated Boolean product distribution $\D_{S}(\vec{p})$ with the associated natural parameter $\vec{z}$ with $\vec{z} \in \ball$, {where $\ball$ is the projection set of \Cref{lem:ball}}. Then $\vec{H}_{ \overline{\ell}}$ is $\lambda_{\vec{z}}$-strongly convex, {where $\lambda_{\vec{z}} = \poly(\alpha)^{\poly(1/ \lambda)},$} where $\lambda$ is introduced in \Cref{as:anticon}.
\end{lemma}
\begin{proof}
We have that $\vec{H}_{ \overline{\ell}} = \text{Cov}_{\vec{x} \sim \D_{S}(\vec{p})}[\vec{x}, \vec{x}]$. We will call this matrix $C_{\vec{p}}$. Then, we have that
\begin{equation*}
     C_{\vec{p}} = \E_{\vec{x} \sim \D_{S}(\vec{p})} \Big[(\vec{x} - \E_{\vec{y} \sim \D_{S}(\vec{p})}[\vec{y}])(\vec{x} - \E_{\vec{y} \sim \D_{S}(\vec{p})}[\vec{y}])^{T} \Big]\,.
\end{equation*}
For arbitrary vector $\vec{v} \in \mathbb{R}^{d}$ with $\|\vec{v}\|_{2} = 1$, we have that to show that
\begin{equation*}
     \vec{v}^{T}C_{\vec{p}} \vec{v} > 0\,.
\end{equation*}
Let us set $\vec{m} = \E_{\vec{y} \sim \D_{S}(\vec{p})}[\vec{y}]$.
Note that
\begin{equation*}
    \vec{v}^{T}C_{\vec{p}} \vec{v} = \E_{\vec{x} \sim \D_{S}(\vec{p})}[p_{v}(\vec{x})] \,,
\end{equation*}
where, after some algebraic manipulation, we can get:
\begin{equation*}
    p_{v}(\vec{x}) = \sum_{j=1}^{d} v_{j}(x_{j}-m_{j}) \sum_{i=1}^{d}v_{i}(x_{i}-m_{i}) = (\vec{v}^{T}(\vec{x}-\vec{m}))^{2}\,.
\end{equation*}
 For the distribution $\D_{S}(\vec{p})$, \Cref{as:anticon} holds (using \Cref{lem:aconc-ball}, since the respective natural parameters $\vec z$ lie inside the ball $\ball$) with a {positive constant $\lambda_{\vec{p}} = \poly(\alpha)^{\poly(1/\lambda)}$}. Specifically, setting $\vec{w} = \vec{v} $ and $ c = \vec{v}^{T}\vec{m}$, \Cref{as:anticon} implies that there exists a positive constant $\lambda_{\vec{p}}$ such that:
\begin{equation*}
    \Pr_{\vec{x} \sim \D_{S}(\vec{p})} \Big [ |\vec{v}^{T}\vec{x}-c| > \lambda_{\vec{p}} \Big ] \geq \lambda_{ \vec{p}}\,.
\end{equation*}
Hence, it follows that:
\begin{equation*}
     \vec{v}^{T}C_{\vec{p}} \vec{v} > \lambda_{\vec{p}}^{3} > 0\,,
\end{equation*}
for any arbitrary unit vector $\vec{v} \in \mathbb{R}^{d}$.
\end{proof}
\subsection{Analysis of SGD}
\label{subs:sgd-an}
Up to that point, we have showed that there exists an initial guess, that is the empirical mean vector $\hat{\vec{z}}$ such that there exists a ball $\ball$ of radius {$B_{\alpha,\lambda}$} centered at the $\vec{\hat{z}}$, that contains the true natural parameters $\vec{z}^{\star}$, with high probability. Additionally, every point that falls inside that ball satisfies Assumptions~\ref{as:anticon} and~\ref{as:mass} and that $ \overline{\ell}$ is strongly convex inside $\ball$.

Apart from the previous analysis, in order to provide the theoretical guarantees of the Projected SGD algorithm, we have to show that, at each iteration, the square of the norm of the gradient vector of the $ \overline{\ell}$ is bounded. This is proved in the following lemma.

Let $\vec{v}^{(t)}$ be the gradient of the negative log-likelihood that our SGD algorithm computes at step $t$. We remind the reader that $\vec{v}^{(t)} = -\vec{x}^{(t)} + \vec{y}$ (see \Cref{algo:sgd}). 

\medskip
\begin{lemma} 
[Bounded Variance Step]
\label{lem:bound-var}
Let $\vec{z}^{\star} \in \mathbb{R}^{d}$ be the true natural parameter vector and let $\vec{z}$ be the guess after step $t-1$ according to which the gradient is computed. Assume that $\vec{z}$ and $\vec{z}^{\star}$ lie inside the ball $\ball$  and that $\min \{ \D(\vec{z};S),\D(\vec{z}^{\star};S)\} \geq \beta$. Then, we have that:
\begin{equation*}
    \E \Big [\| \vec{v}^{(t)} \|_{2}^{2} \Big ] \leq \frac{4d}{\beta}\,.
\end{equation*}
\end{lemma}
\begin{proof}
Let $\vec{p}$ (resp. $\vec{p}^{\star})$ be the corresponding mean parameter vector of the natural parameter vector $\vec{z}$ (resp. $\vec{z}^{\star})$. According to line 8 of the SGD \Cref{algo:sgd} and the \Cref{eq:like}, we have that
\[
\E \Big[ \| \vec v^{(t)} \|_{2}^{2} \Big] =
    \E_{\vec{x} \sim \D_{S}(\vec{p}^{\star})} \left [\E_{\vec{y} \sim \D_{S}(\vec{p})} \| \vec{x} - \vec{y} \|_{2}^{2} \right]\,,
\]
and hence
\begin{equation} \label{eq:norm}
    \E \Big[ \| \vec v^{(t)} \|_{2}^{2} \Big]  \leq 
    2\E_{\vec{x} \sim \D_{S}(\vec{p}^{\star})} \Big[ \| \vec{x} \|_{2}^{2} \Big] + 2\E_{\vec{y} \sim \D_{S}(\vec{p})} \Big [ \| \vec{y} \|_{2}^{2} \Big]\,.
\end{equation}
Now, since the measure of $S$ is greater than $\beta$ for both parameter vectors and since both parameters lie inside the ball, we can appropriately bound the above quantity. Observe that:
\begin{equation*}
    \E_{\vec{y} \sim \D_{S}(\vec{p})}\Big [ \|\vec{y} \|_{2}^{2} \Big] \leq \frac{1}{\beta} \E_{\vec{y} \sim \D(\vec{p})} \Big [ \|\vec{y}\|_{2}^{2}\Big] \leq \frac{d}{\beta}\,.
\end{equation*}
Similarly, we have that:
\begin{equation*}
    \E_{\vec{x} \sim \D_{S}(\vec{p}^{\star})}
    \Big[ \| \vec{x} \|_{2}^{2} \Big] \leq \frac{d}{\beta}\,.
\end{equation*}
The result follows by combining the two inequalities to \Cref{eq:norm}.
 \end{proof}

Let $ \overline{\ell}$ be the negative log-likelihood for the population model. 
We present a folklore SGD theorem. The formulation we use is from \cite{shalev2014understanding}.
\medskip
\begin{fact} \label{thm:sgd-main}
Let $f = \overline{\ell}$. Assume that $f$ is $\mu$-strongly convex, that $\E[\vec{v}^{(t)} | \mathbf{z}^{(t-1)}] \in \partial f(\mathbf{z}^{(t-1)})$  and that $\E \Big [ \| \vec{v}^{(t)} \|_{2}^{2} \Big ] \leq \rho^{2}$. Let $\mathbf{z}^{\star} \in \argmin_{z \in \mathcal{B}} f(\mathbf{z})$ be an optimal solution. Then,
\begin{equation*}
    \E[f(\overline{\mathbf{z}})] - f(\mathbf{z}^{\star}) \leq \frac{\rho^{2}}{2\mu M} \cdot (1 + \ln M)\,,
\end{equation*}
where $\overline{\mathbf{z}}$ is the output of the SGD \Cref{algo:sgd}.
\end{fact}


\medskip
As an application of \Cref{thm:sgd-main} and  \Cref{lem:bound-var}, we obtain directly the following result.
\begin{lemma} \label{lem:sgd-thm}
Let $\vec{z}^{\star}$ be the true parameters of our model, $f = \overline{\ell}, $ $\beta = \min_{\vec{z} \in \mathcal{B}}D(\vec{z};S), \mu \geq \min_{\vec{z} \in \mathcal{B}} \lambda_{\vec{z}}$, then there exists a universal constant $C > 0$ such that
\begin{equation*}
    \E[f(\overline{\vec{z}})] - f(\vec{z}^{\star}) \leq \frac{C d}{ \beta \mu M} \cdot (1 + \ln M)\,.
\end{equation*}
\end{lemma}

We are now ready to prove our main \Cref{thm:sgd}. \begin{proof}
Using \Cref{lem:sgd-thm} and applying Markov's inequality, it follows that:
\begin{equation*}
    \Pr \Big[ f(\overline{\vec{z}}) - f(\vec{z}^{\star}) \geq \frac{3 C d}{ \beta \mu M} \cdot (1 + \ln M) \Big] \leq \frac{1}{3}\,.
\end{equation*}
We can amplify the probability of success to $1-\delta$ by repeating $N = \ln(1/\delta)$ independently from scratch the SGD procedure and keeping the estimation that achieves the maximum log-likelihood value.  The procedure is completely similar to the proof of Theorem 1 of \cite{DGTZ18} and we repeat it here for completeness. Let $\mathcal{E}$ be the set of our $N$ estimates. The optimal estimate would be $\vec{\widetilde{z}} = \argmin_{\vec{z} \in \mathcal{E}} \overline{\ell}(\vec{z})$, but we cannot compute exactly $f = \overline{\ell}$. Using the Markov's inequality, we get that, with probability at least $1-\delta$, at least $2/3$ of our estimates satisfy
\begin{equation*}
   f(\vec{z}) - f(\vec{z}^{\star}) \leq \frac{3 C d}{ \beta \mu M} \cdot (1 + \ln M)\,.
\end{equation*}
Let us set $\zeta := \frac{3 C d}{ \beta \mu M}(1 + \ln M)$. As we will see, using the strong convexity property, we get that $f(\vec{z}) - f({\vec{z}^{\star}})$, implies $\|\vec{z} - \vec{z}^{\star}\|_{2} \leq c \zeta$, for some $c$. Hence, with high probability $1-\delta$ for at least $2/3$ of our estimations,
the $L_{2}$ norm is at most $2c\zeta$. So, we can set appropriately the value of $\vec{\widetilde{z}}$ 
to be a point that is at least $2c\zeta$ close to more that the half of our $N$ estimations. That value will satisfy $f(\widetilde{\vec{z}}) - f(\vec{z}^{\star}) \leq \zeta$. {Now, using Lemmata~\ref{lem:aconc-ball} and~\ref{lem:mass-ball}, one can show that $\beta \geq \alpha^{\poly(1/\lambda)}$ and $\mu \geq \poly(\alpha)^{\poly(1/\lambda)}$,
where $\alpha$ is the constant of \Cref{as:mass} and $\lambda$ is the parameter of \Cref{as:anticon}.
This leads to the following statement: With probability at least $1-\delta$, we have that:
$f(\widetilde{\vec{z}}) - f(\truez) \leq \frac{d \cdot  \poly(1/\alpha)^{ \poly(1/\lambda)}}{M}(1 + \ln M)$.} Now, we can use the Lemma 13.5 of \cite{shalev2014understanding} about strong convexity:
\medskip
\begin{fact}
\label{lem:str}
If $f$ is $\mu$-strongly convex and $\vec{z}^{\star}$ is a minimizer of $f$, then, for any $\vec{z}$, it holds that:
\begin{equation*}
    f(\vec{z}) - f(\vec{z}^{\star}) \geq \frac{\mu}{2}\|\vec{z}-\vec{z}^{\star}\|_{2}^{2}\,.
\end{equation*}
\end{fact}
\noindent Using this result, we can get 
\begin{equation*}
    {\|\vec{\widetilde{z}}-\vec{z}^{\star}\|_{2} \lesssim \sqrt{\frac{d \cdot \exp{(1/\lambda)}}{M} \cdot (1 + \ln M)}}\,,
\end{equation*}
where {we consider that $\alpha = O(1)$}. 
{For $N = \ln(1/\delta)$ and $M \gtrsim_{\alpha} \exp(1/\lambda) \cdot \widetilde{O} \left (\frac{d}{\epsilon^{2}} \right)$, the result follows.}
 \end{proof}

\paragraph{Acknowledgments}
{We would like to thank Rohan Chauhan and Ioannis Panageas for pointing out the correct dependence of our SGD algorithm on the anti-concentration parameter.}
We thank the anonymous reviewers for useful remarks and comments on the presentation of our manuscript.
This work was supported by the Hellenic Foundation for Research and Innovation (H.F.R.I.) under the ``First Call for H.F.R.I. Research Projects to support Faculty members and Researchers and the procurement of high-cost research equipment grant'',  project BALSAM, HFRI-FM17-1424.



\begingroup
    \bibliography{references.bib}
\endgroup

\appendix
\section{Appendix: Deferred Proofs}
\label{app:ineq}
In this section, we provide the proof of \Cref{prop:kl-norm-ineq}.

\begin{proof}
We define the pair of functions on the space $(p,q) \in (0,1)^{2}$:
\begin{equation*}
    f(p,q) = p\ln \frac{p}{q} + (1-p)\ln \frac{1-p}{1-q}
\end{equation*}
and
\begin{equation*}
    g(p,q) = \Big( \ln \frac{p}{1-p} - \ln \frac{q}{1-q} \Big )^{2}\,.
\end{equation*}

Both functions have a root at $p=q=1/2.$ Notice that $g$ is symmetric. Fix $q.$ We will denote with $f_{q}$ (resp. $g_{q}$) the projection of $f$ (resp. $g)$ in the $p$-space, having fixed $q.$ Then, $f_{q}(q) = g_{q}(q) = 0$ is the unique root for $p \in (0,1).$ Let $h(p) = f_{q}(p) - g_{q}(p).$ We claim that $h$ has a unique root at $q$ for $p \in (0,1).$ The derivate of $h$ with respect to $p$ is equal to:
\begin{equation*}
    \frac{dh}{dp} = \ln\bigg(\frac{p(1-q)}{q(1-p)}\bigg) \Big (1 - \frac{2}{p(1-p)} \Big )   \,.
\end{equation*}
Notice that: $ 1 - \frac{2}{p(1-p)} < 0 \hspace{1mm} \forall p \in (0,1)$ and that:
\begin{equation*}
    \ln\bigg(\frac{p(1-q)}{q(1-p)}\bigg) = \begin{cases} 
      < 0 & \text{ for } p < q\,, \\
      0 &\text{ for } p=q\,, \\
      > 0 & \text{ for } p > q\,. 
   \end{cases}
\end{equation*}

Hence, $h'(q) = 0$ and, hence, $h$ is strictly increasing for $p < q$ and $h$ is strictly decreasing for $p > q.$
Also, $p=q$ is the unique solution of the equation $h(p) = 0$ for $p \in (0,1).$ 

For $p < q \Rightarrow h(p) < 0 \Rightarrow f_{q}(p) < g_{q}(p)$ and for $p > q \Rightarrow h(p) < 0 \Rightarrow f_{q}(p) < g_{q}(p).$ So, the desired inequality holds for the arbitrary fixed $q \in (0,1)$. Hence, the inequality follows for every $p,q \in (0,1).$
\end{proof}

\end{document}